\documentclass{article}
\usepackage{authblk}
\usepackage[margin=1in]{geometry}
\pdfoutput=1
\usepackage[T1]{fontenc}    
\usepackage{url}            
\usepackage{booktabs}       
\usepackage{amsfonts}       
\usepackage{nicefrac}       
\usepackage{microtype}      

\usepackage[usenames,dvipsnames]{color}
\usepackage[pdftex]{graphicx} 
\usepackage{verbatim}
\usepackage{algorithmic, algorithm}
\usepackage{amsmath}
\usepackage{amssymb}
\usepackage{amsthm}
\usepackage{epsfig}
\usepackage{wrapfig}
\usepackage{subcaption}

\newtheorem{prop}{Proposition}
\newtheorem{lem}{Lemma}
\newtheorem{cor}{Corollary}

\newtheorem{defin}{Definition}

\newtheorem{rem}{Remark}
\newtheorem{asm}{Assumption}

\newcommand{\matI}{{\bf I}}

\newcommand{\matX}{{\bf X}}

\newcommand{\vece}{{\bf e}}
\newcommand{\vecx}{{\bf x}}
\newcommand{\vecy}{{\bf y}}
\newcommand{\vecw}{{\bf w}}

\newcommand{\vecv}{{\bf v}}

\newcommand{\vecr}{{\bf r}}

\newcommand{\pnorm}[2]{{\Vert #1 \Vert} _{#2}}
\newcommand{\Lovasz}{Lov\'asz }

\newcommand{\pdim}{p}

\newcommand{\param}{\bbeta}

\newcommand{\thetahat}{\widehat{\param}}
\newcommand{\thetastar}{\param^*}

\newcommand{\thetagreedy}{\thetahat}
\newcommand{\thetabar}{\bar{\param}} 

\newcommand{\sparsity}{k}
\newcommand{\response}{y_i}
\newcommand{\feature}{\vecx_i}

\newcommand{\numobs}{n}
\newcommand{\inprod}[2]{\langle#1, #2\rangle}
\newcommand{\powerset}{2^{[\pdim]}}

\newcommand{\functionvector}{l}
\newcommand{\functionset}{f}

\mathchardef\mhyphen="2D

\DeclareMathOperator*{\argmax}{arg\,max}
\DeclareMathOperator*{\supp}{supp}

\newcommand{\vertiii}[1]{{\left\vert\kern-0.25ex\left\vert\kern-0.25ex\left\vert #1
    \right\vert\kern-0.25ex\right\vert\kern-0.25ex\right\vert}}

\newcommand{\grad}[1]{ \nabla_{#1} }

\newcommand{\vect}[1]{{\boldsymbol{#1}}}

\def\bbeta{\vect{\beta}}

\def\bv{{\mathbf{v}}}

\def\bx{{\mathbf{x}}}
\def\by{{\mathbf{y}}}

\def\bA{{\mathbf{A}}}
\def\bB{{\mathbf{B}}}
\def\bC{{\mathbf{C}}}
\def\bD{{\mathbf{D}}}

\def\bH{{\mathbf{H}}}
\def\bI{{\mathbf{I}}}

\def\bX{{\mathbf{X}}}
\def\bY{{\mathbf{Y}}}

\def\bold0{{\mathbf{0}}}

\def\bbB{{\mathbb{B}}}

\def\bbE{{\mathbb{E}}}

\def\bbR{{\mathbb{R}}}

\def\cP{\mathcal{P}}

\def\sfA{\mathsf{A}}
\def\sfB{\mathsf{B}}

\def\sfL{\mathsf{L}}

\def\sfS{\mathsf{S}}
\def\sfT{\mathsf{T}}
\def\sfU{\mathsf{U}}

\newtheorem{theorem}{Theorem}

\renewcommand{\text}[1]{{\textnormal{#1}}}

\title{Restricted Strong Convexity \\ Implies Weak Submodularity\footnote{Research supported by NSF Grants CCF 1344179, 1344364, 1407278, 1422549, DMS 1723052, IIS 1421729, and ARO YIP W911NF-14-1-0258.}}
\author[1]{Ethan R. Elenberg} 
\author[1]{Rajiv Khanna}
\author[1]{Alexandros G. Dimakis}
\author[2]{Sahand Negahban}
\affil[1]{Department of Electrical and Computer Engineering \authorcr The University of Texas at Austin \authorcr \texttt{\{elenberg,\,rajivak\}@utexas.edu}, \texttt{dimakis@austin.utexas.edu}}
\affil[2]{Department of Statistics \authorcr Yale Univeristy \authorcr \texttt{sahand.negahban@yale.edu}}
 
\usepackage{hyperref}
\begin{document}
\maketitle

\begin{abstract}
We connect high-dimensional subset selection and submodular maximization. Our results extend the work of Das and Kempe (2011) from the setting of linear regression to arbitrary objective functions. 
For greedy feature selection, this connection allows us to obtain strong multiplicative performance bounds on several methods without statistical modeling assumptions. We also derive recovery guarantees of this form under standard assumptions. 
Our work shows that greedy algorithms perform within a constant factor from the best possible subset-selection solution for a broad class of general objective functions. Our methods allow a direct control over the number of obtained features as opposed to regularization parameters that only implicitly control sparsity.
Our proof technique uses the concept of \textit{weak submodularity} initially defined by Das and Kempe. We draw a connection between convex analysis and submodular set function theory which may be of independent interest for other statistical learning applications that have combinatorial structure.
\end{abstract}

\section{Introduction}
\label{sec:intro}
Sparse modeling is central in modern data analysis and high-dimensional statistics since it provides interpretability and robustness. 
Given a large set of $\pdim$ features we wish to build a model using only a small subset of $\sparsity$ features: the central combinatorial question is how to choose the optimal feature subset. Specifically, we are interested in optimizing over sparse parameter vectors $\param$ and consider problems of the form: 
\begin{equation}
\label{eq:opt}
\thetabar^k \in \argmax_{\param : \|\param\|_0 \leq k} \functionvector(\param) \; ,
\end{equation}
for some function $\functionvector(\cdot)$. This is a very general framework:  the function 
$\functionvector(\cdot)$ can be a linear regression $R^2$ objective, a generalized linear model (GLM) likelihood, a graphical model learning objective, or an arbitrary $M$-estimator~\cite{Negahban2012journal}. This \textit{subset selection} problem is NP-hard~\cite{natarajan1995sparse} even for the sparse linear regression objective, and a vast body of literature has analyzed different approximation algorithms under various assumptions.

The Restricted Isometry Property (RIP) and the (closely related) Restricted Eigenvalue property are conditions on
$\functionvector(\param)$ that allow convex relaxations and greedy algorithms to solve the subset selection problem within provable approximation guarantees. In parallel work, several authors have demonstrated that the subset selection problem can be connected to submodular optimization~\cite{Hoi2006active,Wei2013submod,Wei2015,Bach_Monograph} and that greedy algorithms are widely used for iteratively building good feature sets. 

The mathematical connection between submodularity and RIP was made explicit by Das and Kempe~\cite{Das2011} for linear regression. 
Specifically, they showed that when $\functionvector(\param)$ is the $R^2$ objective, it satisfies a weak form of submodularity when the linear measurements satisfy RIP. Note that for a given set of features $\sfS$, the function $\functionvector(\param_{\sfS})$ with support restricted to $\sfS$ can be thought of as a set function and this is key in this framework. Using this novel concept of \textit{weak submodularity} they established strong multiplicative bounds on the performance of greedy algorithms for subset selection and dictionary selection. Work by Bach~\cite{Bach_Monograph} in the linear regression setting discusses the notion of \emph{suppressors}; however, that condition is stronger than the weak submodularity assumption.
Krause and Cevher \cite{Krause2010} draw similar connections between submodularity and sparse regression, but they require a much stronger coherence-based assumption.

In this paper we extend this machinery beyond linear regression, to any function $\functionvector(\param)$.
To achieve this we need the proper generalization of the Restricted Eigenvalue and RIP conditions for arbitrary functions. 
This was obtained by Negahban, et al.~\cite{Negahban2012journal} and is called \textit{Restricted Strong Convexity}.
The title of this paper should now be clear:  we show that any objective function that satisfies \emph{restricted strong convexity} 
(and a natural smoothness assumption) of \cite{Negahban2012journal} must be weakly submodular. 

We establish multiplicative approximation bounds on the performance of greedy algorithms, including (generalized) Orthogonal Matching Pursuit and Forward Stepwise Regression, for general likelihood functions using our connection. 
To the best of our knowledge, this is the first analysis of connecting
a form of submodularity to the objective function's strong concavity
and smoothness.
Our approach provides sharp approximation bounds in any setting where these fundamental structural properties are well-understood, \textit{e.g.} generalized linear models.

Contrary to prior work we require no assumptions on the sparsity of the underlying problem.
Rather, we obtain a deterministic result establishing multiplicative approximation guarantees from the best-case sparse solution. 
Our results improve over previous work by providing bounds on a solution that is guaranteed to match the desired sparsity.
Convex methods such as $\ell_1$ regularized objectives require strong assumptions on the model, such as the irrepresentability conditions on the feature vectors, in order to provide exact sparsity guarantees on the recovered solution.

Our main result is that for any function $l(\cdot)$ that satisfies $M$-restricted smoothness (RSM) and $m$-restricted strong convexity (RSC), the set function $f(\sfS) = -l(\bbeta_\sfS)$ is weakly submodular with parameter 
$\gamma \geq \nicefrac{m}{M}$. The parameters $M$, $m$, and $\gamma$ are defined formally in Section \ref{sec:prelim}.
We use this result to analyze three greedy algorithms, each progressively better but more computationally intensive: 
the Oblivious (or Marginal Regression) algorithm computes for each feature the increase in objective and keeps the $k$ individually best features. 
Orthogonal Matching Pursuit (OMP) greedily adds one feature at a time by picking the feature with the largest inner product with the function gradient. The gradient is the correct generalization of the residual error used in linear regression OMP. 	
Finally, the most sophisticated algorithm is Forward Stepwise Regression: it adds one feature at a time by re-fitting the model repeatedly and keeping the feature that best improves the objective function at each step. 
We obtain the following performance bounds: 
\begin{itemize}
\item The Oblivious algorithm produces
a ($\nicefrac{\gamma}{k}$)-approximation 
to the best $k$-subset after $k$ steps.
\item Orthogonal Matching Pursuit produces a 
($1 - e^{-\nicefrac{m}{M}}$)-approximation 
to the best $k$-subset after $k$ steps.
\item Forward Stepwise Regression produces a ($1 - e^{-\gamma}$)-approximation to the best $k$-subset after $k$ steps.
\end{itemize}
We also show that if Forward Stepwise Regression is used to select more than $k$ features, we can approximate the best $k$-sparse feature performance within an \textit{arbitrary} accuracy. Finally, under additional assumptions, we derive statistical guarantees for convergence of the greedily selected parameter to the optimal sparse solution. Note that our results yield stronger performance guarantees even for cases that have been previously studied under the same assumptions. For example, for linear regression we obtain a better exponent in the approximation factor of OMP compared to previous state-of-the-art \cite{Das2011} (see Remark~\ref{rem:ompImprove}).

One implication of our work is that weak submodularity seems to be a sharper technical tool than RSC, as
any function satisfying the latter also satisfies the former.
Das and Kempe~\cite{Das2011} noted that it is easy to find problems which satisfy weak submodularity but not RSC, emphasizing the limitations of 
spectral techniques versus submodularity. 
We show this holds beyond linear regression, for any likelihood function.

Our connection between restricted strong convexity and weak submodularity has many benefits. First, the weak submodularity framework can now be used to develop theory for additional problems where spectral conditions would be overly restrictive or unwarranted. In our ongoing work, for example, we show that RSC assumptions play an important role in characterizing distributed greedy maximization \cite{Khanna2017} and rank constrained matrix optimization \cite{Khanna2017matrix}. Additionally, \cite{streamingWS} studies the role of RSC in streaming feature selection, as well as interpretability of black-box neural network models. Second, this framework allows statisticians to draw from recent advances (and classical results) in the field of submodular set function theory to provide general guarantees on the performance of greedy algorithms.

\paragraph{Related Work}
There have been a wide range of techniques developed for solving problems with sparsity constraints. These include using the Lasso, greedy selection methods 
(such as Forward Stagewise/Stepwise Regressions \cite{Mazumder2015boostingpreprint}, OMP, and CoSaMP \cite{cosamp}), 
forward-backward methods \cite{Jalali2011,Liu2014}, Pareto optimization \cite{Qian2015Pareto}, exponentially weighted aggregation \cite{Rigollet2011}, and truncated gradient methods~\cite{Jain2014iht}. Under the restricted strong convexity and smoothness assumptions that will be outlined in the next section, forward-backward methods can in fact recover the correct support
of the optimal set of parameters under an assumption on the smallest value of the optimal variable as it relates to the gradient. In contrast, the results derived in our setting for sparse GLMs allow one to provide recovery guarantees at various sparsity levels regardless of the optimal solution, with only information on the desired sparsity level and the RSC and RSM parameters. This is again in contrast to the other work that also needs information on the smallest nonzero value in the optimal set of coefficients, as well as an upper bound on the gradient of the objective at this optimal set. 

Focusing explicitly on OMP, most previous results require the strong RIP assumption, whereas we only require the weaker RSC and RSM assumptions. In our setting of arbitrary model conditions, OMP requires RIP as highlighted in Corollary 2.1 of Zhang~\cite{Zhang11}. However, we do note that under certain stochastic assumptions, for instance independent noise, the results established in those works can provide sharper guarantees with respect to the number of samples required by a factor on the order of $\log \left [ (k \log \pdim) /n \right ]$ (See Section \ref{sec:stats}). Nevertheless, we emphasize that our results apply under arbitrary assumptions on the noise and use only RSC and RSM assumptions.

In \cite{Das2011}, Das and Kempe's framework optimizes the goodness of fit parameter $R^2$ in linear regression. We derive similar results without relying on the closed-form solution to least squares. 
Greedy algorithms are prevalent in compressed sensing literature \cite{cosamp} and statistical learning theory \cite{Barron2008}. 
Greedy methods for sparsity constrained regression were analyzed in 
\cite{Jalali2011,Liu2014,Jain2014iht,Lozano2011,Tewari2011,Yuan2014} 
under assumptions similar to ours but without connections to submodularity. 
Convergence guarantees for $\ell_1$ regularized regression were given for exponential families in \cite{Kakade2010}, and for general nonlinear functions in \cite{Yang2016}. However, the latter requires additional assumptions such as knowledge of the nonlinearity and bounds on the loss function's derivatives, which can again be derived under appropriate stochasticity and model assumptions.

Classical results on submodular optimization~\cite{Nemhauser1978,Conforti1984} typically do not scale to large-scale applications. Therefore, several recent algorithms improve efficiency at the expense of slightly weaker guarantees~\cite{MirzasoleimanLazy,Barbosa2015,PanParallelDoubleGreedy,Khanna2017,streamingWS}. Submodularity has been used recently in the context of active learning \cite{Hoi2006active,Wei2015}. In this setup, the task is to select predictive data points instead of features.
Recently, \cite{Altschuler2016} and \cite{SingerApproximate2016} obtained constant factor guarantees for greedy algorithms using techniques from submodularity even though the problems considered were not strictly submodular. For maximizing weakly submodular functions subject to more general matroid constraints, \cite{Chen2017} recently proved that a randomized greedy forward stepwise algorithm has a constant factor approximation guarantee.

There are deep connections between convexity and submodularity~\cite{Bach_Monograph}. For example, the convex closure of a submodular function can be tractably computed as its \Lovasz extension~\cite{Lovasz1983}. This connection is fundamental to providing polynomial-time minimization algorithms for submodular set functions~\cite{fujishige2005submodular,Hazan2012online}. Similarly, another continuous extension of set functions, called the multilinear extension is vital for algorithmic development of constant factor approximation guarantees for submodular maximization~\cite{Buchbinder2016b}.
A more detailed study of convexity/concavity-like properties of submodular functions was presented in~\cite{Iyer2015}. More recent works exploit similar connections to provide constant factor approximation guarantees for a class of non-convex functions~\cite{Bian2017,Eghbali2016,Hassani2017}.

\section{Preliminaries}\label{sec:prelim}
First we collect some notation that will be used throughout the remainder of this paper.
Sets are represented by sans script fonts \textit{e.g.} $\sfA, \sfB$. Vectors are represented using lower case bold letters \textit{e.g.} $\bx,\by$, and matrices are represented using upper case bold letters \textit{e.g.} $\bX,\bY$. The $i$-th column of $\matX$ is denoted $\matX_i$. Non-bold face letters are used for scalars \textit{e.g.} $j,M,r$ and function names \textit{e.g.} $f(\cdot)$. The transpose of a vector or a matrix is represented by $\top$ \textit{e.g.} $\bX^\top$. For any vector $\vecv$, define $\|\vecv\|_{2,k} = \sqrt{\sum_{i=1}^k v_{(i)}^2}$, where $v_{(i)}$ represent the order statistics of $\vecv$ in decreasing order. Define $[p]:=\{1,2,\ldots, p\}$. For simplicity, we assume a set function defined on a ground set of size $p$ has domain $\powerset$. For singleton sets, we write $f(j) := f(\{j\})$. 

Recall that a set function $f(\cdot) : \powerset \rightarrow \bbR$ is called submodular if and only if for all $\sfA,\sfB \subseteq [p]$, $f(\sfA) + f(\sfB) \geq f(\sfA \cup \sfB) + f(\sfA \cap \sfB)$.
Intuitively, submodular functions have a \textit{diminishing returns} property. This becomes clear when $\sfA$ and $\sfB$ are disjoint: the sets have less value taken together than they have individually. 
We also state an equivalent definition:
\begin{defin}[Proposition 2.3 in \cite{Bach_Monograph}]
$f(\cdot)$ is submodular if for all $\sfA \subseteq [p]$ and $j,k \in [p] \backslash \sfA$ ,
\begin{align*}
f(\sfA \cup \{k\}) - f(\sfA) \geq f(\sfA \cup \{j,  k\}) - f(\sfA \cup \{j\}) .
\end{align*}
\end{defin}

The function is called normalized if $f(\emptyset) = 0$ and monotone if and only if $f(\sfA) \leq f(\sfB)$ for all $\sfA \subseteq \sfB$.
A seminal result by Nemhauser, et al.~\cite{Nemhauser1978} shows that greedy maximization of a monotone, submodular function (Algorithm \ref{alg:greedy} in Section \ref{subsec:algorithms}) returns a set with value within a factor of $(1 - \nicefrac{1}{e})$ from the optimum set of the same size. This has been the starting point for several algorithmic advances for large-scale combinatorial optimization, including stochastic, distributed, and streaming algorithms \cite{MirzasoleimanLazy,Barbosa2015,Khanna2017,streamingWS}.

Next, we define the submodularity ratio of a monotone set function.
\begin{defin}[Submodularity Ratio~\cite{Das2011}, Weak Submodularity]\label{def:submodularityRatio}
	Let $\sfS, \sfL \subset [p]$ be two disjoint sets, and $f(\cdot) : \powerset \rightarrow \bbR$. The submodularity ratio of $\sfL$ with respect to $\sfS$ is given by
	\begin{align}
	\gamma_{\sfL,\sfS} := \frac{\sum_{j\in \sfS} \left[f(\sfL \cup \{j\}) - f(\sfL) \right]}{f(\sfL \cup \sfS) - f(\sfL)} \label{eq:submodDef} .
	\end{align}
	The submodularity ratio of a set $\sfU$ with respect to an integer $k$ is given by
	\begin{align}
	\gamma_{\sfU,k} := \min_{\substack{\sfL,\sfS :  \sfL \cap \sfS = \emptyset ,\\ \sfL \subseteq \sfU,  |\sfS| \leq k}} \gamma_{\sfL,\sfS} .
	\end{align}
	Let $\gamma > 0$. We call a function $\gamma$-weakly submodular at a set $\sfU$ and an integer $k$ if $\gamma_{\sfU,k}\geq \gamma$.
\end{defin}

It is straightforward to show that $f(\cdot)$ is submodular if and only if $\gamma_{\sfL,\sfS} \geq 1$ for all sets $\sfL$ and $\sfS$. In our application, $0 < \gamma_{\sfL,\sfS} \leq 1$ which provides a notion of \textit{weak submodularity} in the sense that even though the function is not submodular, it still provides provable bounds of performance of greedy selections.

Next we define the restricted versions of strong concavity and smoothness, consistent with~\cite{Negahban2012journal, loh2015}.
\begin{defin}[Restricted Strong Concavity, Restricted Smoothness] \label{def:RSCRSM}
	A function $l: \bbR^p \rightarrow \bbR$ is said to be restricted strong concave with parameter $m_\Omega$ and restricted smooth with parameter $M_\Omega$ on a domain $\Omega \subset \bbR^p \times \bbR^p$ if for all $(\bx,\by) \in \Omega$,
	\begin{align*}
	- \frac{m_\Omega}{2} \pnorm{\by -\bx}{2}^2 &\geq l(\by) - l(\bx) - \langle \nabla l(\bx) , \by - \bx \rangle \geq  - \frac{M_\Omega}{2} \pnorm{\by - \bx}{2}^2 . \end{align*}
\end{defin}

\begin{rem}
	If a function $l(\cdot)$ has restricted strong concavity parameter $m$, then its negative $-l(\cdot)$ has restricted strong convexity parameter $m$. In the sequel, we will use these properties interchangeably for maximum likelihood estimation where $l(\cdot)$ is the log-likelihood function and $- l(\cdot)$ is the data fit loss.
\end{rem}

If $\Omega^{\prime} \subseteq \Omega$, then $M_{\Omega^{\prime}} \leq M_{\Omega}$ and $m_{\Omega^{\prime}} \geq m_{\Omega}$. With slight abuse of notation, unless stated otherwise let $(m_k,M_k)$ denote the RSC and RSM parameters on the domain $\Omega_{k}$ of all pairs of 
$k$-sparse vectors that differ in at most $k$ entries, \textit{i.e.} $\Omega_k := \{(\vecx,\vecy):  \pnorm{\vecx}{0} \leq k, \enspace \pnorm{\vecy}{0} \leq k, \enspace  \pnorm{\vecx-\vecy}{0} \leq k\}$.
If $j \leq k$, then $M_j \leq M_k$ and $m_j \geq m_k$.
In addition, denote $\tilde{\Omega}_k := \{(\vecx,\vecy): \pnorm{\vecx}{0} \leq k, \enspace \pnorm{\vecy}{0} \leq k, \enspace \pnorm{\vecx-\vecy}{0} \leq 1\}$ with corresponding smoothness parameter $\tilde{M}_k$, which is clearly greater than or equal to $\tilde{M}_1 = M_1$.

\subsection{Sparsity Constrained Generalized Linear Regression} 
\label{subsec:regression}
Due to its combinatorial nature, there has been a tremendous amount of effort in developing computationally tractable and fundamentally sound methods to solve the subset selection problem approximately. In this section we provide background on various problems that arise in subset selection. Our focus here will be on sparse regression problems.
We will assume that we obtain $\numobs$ observations of the form $(\feature,y_i)$. For now we make no assumptions regarding how the data is generated, but wish to model the interaction between $\feature \in \bbR^p$ and $y_i \in \bbR$ as
\begin{equation*}
\response = g(\inprod{\feature}{\thetastar}) + \text{noise} ,
\end{equation*}
for some known link function $g$ and a sparse vector $\thetastar$. 
Each feature observation is a row in the $\numobs \times p$ design matrix $\matX$.
The above is called a generalized linear model, or GLM, and arises as the maximum likelihood estimate of data drawn from a canonical exponential family, \textit{i.e.} normal, Bernoulli, Dirichlet, negative binomial, etc.~\cite{Brown86}. Another interpretation is in minimizing the average Bregman divergence between the response $\response$ and the mean parameter $\inprod{\feature}{\beta}$. There has been a large body of literature studying this method's statistical properties. These include establishing sparsistency, parameter consistency, and 
prediction error~\cite{Kakade2010,Geer08,Negahban2012journal}.
We refer the reader to the standard literature for more details on GLMs and exponential families~\cite{Brown86,Dey00}.

\subsection{Support Selection Algorithms}\label{subsec:algorithms}

We study general $M$-estimators of the form \eqref{eq:opt}
for some function $\functionvector(\cdot)$. Note that $\functionvector(\cdot)$ will implicitly depend on our specific data set, but we hide that for ease of notation. One common choice of $\functionvector(\cdot)$ is the log-likelihood of a parametric distribution. \cite{Das2011} considers the specific case of maximizing the $R^2$ objective. Through a simple transformation, that is equivalent to maximizing the log-likelihood of the parametric distribution that arises from the model $y_i = \inprod{\feature}{\thetastar}+ w$ where $w \sim N(0,\sigma^2)$. If we let $\thetagreedy^s$ be the $s$-sparse solution derived, and again let $\thetabar^k$ be the best $k$-sparse parameter, then we wish to bound
\begin{equation*}
\functionvector(\thetagreedy^s) \geq (1-\epsilon) \functionvector(\thetabar^k),
\end{equation*}
without any assumptions on the underlying sparsity or a \emph{true} parameter.

For a concave function $\functionvector(\cdot) : \bbR^p \rightarrow \bbR$, we can define an equivalent set function $\bar{\functionset}(\cdot): \powerset \rightarrow \bbR$ so that $\bar{\functionset}(\sfS) = \max_{\supp(\bx) \subseteq \sfS} \functionvector (\bx)$. The problem of support selection for a given integer $k$ is then $\max_{| \sfS | \leq k} \bar{\functionset}(\sfS)$. Recall that a vector is $k$-sparse if it is $0$ on all but $k$ indices. 
We provide approximation guarantees on the \emph{normalized} set function defined as $f(\sfS) = \bar{f}(\sfS) - \bar{f}(\emptyset)$. 
The support selection problem is thus equivalent to finding the $k$-sparse vector $\param$ that maximizes $\functionvector(\param)$: 
\begin{align}
\max_{\sfS : | \sfS | \leq k} f(\sfS) \Leftrightarrow \max_{\substack{\bbeta: \bbeta_{\sfS^c = 0} \\ | \sfS | \leq k}} l(\bbeta) - l(\bold0).
\label{eq:mlfunction}
\end{align}
Let $\bbeta^{(\sfA)}$ denote the $\bbeta$ maximizing $\functionset(\sfA)$, and let $\bbeta^{(\sfA)}_\sfB$ denote $\bbeta^{(\sfA)}$ restricted to the coordinates specified by $\sfB$. We present three support selection strategies for the set function $\functionset(\cdot)$ that are simple to implement and are widely used.

\paragraph{Oblivious Algorithm}
One natural strategy is to select the top $k$ features ranked by their individual improvement over a null model, using a goodness of fit metric such as $R^2$ or 
$p$-value.
This is referred to as the \textit{Oblivious} algorithm, shown as Algorithm \ref{alg:oblivious}. In the context of linear regression, this is simply Marginal Regression. While it is computationally inexpensive and parallelizes easily, the Oblivious algorithm does not account for dependencies or redundancies in the span of features. 

\paragraph{Forward Stepwise Algorithm}
A less extreme greedy approach would check for incremental gain at each step using nested models. This is referred to as the \textit{Forward Stepwise} algorithm, presented as Algorithm \ref{alg:greedy}. Given a set of features $\sfS$ is already selected, choose the feature with largest marginal gain, \textit{i.e.} select $\{j\}$ such that $\sfS \cup \{j\}$ has the most improvement over $\sfS$. All regression coefficients are updated each time a new feature is added. 
In the case of submodular set functions, this returns a solution that is provably within a constant factor of the optimum \cite{Nemhauser1978}.

\begin{algorithm}[H]
    \caption{Oblivious Support Selection}
   \label{alg:oblivious}
\begin{algorithmic}[1]
    \STATE {\bfseries Input:} sparsity parameter $k$, set function $f(\cdot): \powerset \rightarrow \bbR$ 
    \FOR{$i=1\ldots p$}
	    \STATE $\bv[i] \gets f(\{i\})$
    \ENDFOR
	\STATE $\sfS_k \gets$ indices corresponding to the top $k$ values of $\bv$
    \RETURN $\sfS_k$, $f(\sfS_k)$.
\end{algorithmic}
\end{algorithm}
\begin{algorithm}[H]
    \caption{Forward Stepwise Selection}
   \label{alg:greedy}
\begin{algorithmic}[1]
	\STATE {\bfseries Input:} sparsity parameter $k$, set function $f(\cdot): \powerset \rightarrow \bbR$ 
	\STATE $\sfS_0^G \gets \emptyset$
	\FOR{$i=1\ldots k$}
	\STATE $s \gets \arg \max_{j \in [p]\backslash \sfS_{i-1}}  f(\sfS_{i-1}^G \cup \{j\}) - f(\sfS_{i-1}^G)$
	\STATE $\sfS_i^G \gets \sfS_{i-1}^G \cup \{s\}$
	\ENDFOR
	\RETURN $\sfS_k^G$, $f(\sfS_k^G)$.
\end{algorithmic}
\end{algorithm}

\begin{algorithm}[ht]
    \caption{Orthogonal Matching Pursuit}
   \label{alg:omp}
\begin{algorithmic}[1]
	\STATE {\bfseries Input:} sparsity parameter $k$, objective function $l(\cdot): \bbR^p \rightarrow \bbR$ 
	\STATE $\sfS_0^P \gets \emptyset$
	\STATE $\vecr \gets \grad{} l(\mathbf{0})$
	\FOR{$i=1\ldots k$}
		\STATE $s \gets \arg \max_{j}  | \langle \vece_j, \vecr \rangle |$
	\STATE $\sfS_i^P \gets \sfS_{i-1}^P \cup \{s\}$
	\STATE $\bbeta^{(\sfS_i^P)} \gets \argmax_{\bbeta: \supp(\bbeta)  \subseteq \sfS_i^P } l(\bbeta)$
	\STATE $\vecr \gets \grad{} l(\bbeta^{(\sfS_i^P)})$
	\ENDFOR
	\RETURN $\sfS_k^P$, $l(\bbeta^{(\sfS_k^P)})$.
\end{algorithmic}
\end{algorithm}

\paragraph{Generalized OMP}
	Another approach is to choose features which correlate well with the orthogonal complement of what has already been selected. 
	Using \eqref{eq:mlfunction} and an appropriately chosen model, 
	we can define the gradient evaluated at the current parameter $\param$ to be a residual term. In \textit{Orthogonal Matching Pursuit}, features are selected to maximize the inner product with this residual, as shown in line $5$ of Algorithm \ref{alg:omp}. Here $\vece_j$ represents a unit vector with a $1$ in coordinate $j$ and zeros in the other $p-1$ coordinates. OMP requires much less computation than forward stepwise selection, since the feature comparison is done via an $n$-dimensional inner product rather than a regression score. A detailed discussion can be found in \cite{Blumensath2007}.

\section{Approximation Guarantees} 
\label{sec:derivation}
In this section, we derive theoretical lower bounds on the submodularity ratio based on strong concavity and strong smoothness of a function $l(\cdot)$. We show that if the concavity parameter is bounded away from $0$ and the smoothness parameter is finite, then the submodularity ratio is also bounded away from $0$, which allows approximation guarantees for Algorithms \ref{alg:oblivious}--\ref{alg:omp}.  While our proof techniques differ substantially, the outline of this section follows that of~\cite{Das2011} which obtained approximation guarantees for support selections for linear regression. While our results are applicable to general functions, in Appendix~\ref{sec:glmAppendix} we discuss a direct application of maximum likelihood estimation for sparse generalized linear models.

We assume a differentiable function $l: \bbR^p \rightarrow \bbR$. Recall that we can define the equivalent, normalized, monotone set function $f: \powerset \rightarrow \bbR$ for a selected support as $f(\sfS) = \max_{\text{supp}(\bbeta) \subseteq \sfS} l(\bbeta) - l(\bold0)$. 
We will use set functions wherever possible to simplify the notation.

We now present our main result as Theorem~\ref{thm:ratioBound}, a bound on a function's submodularity ratio $\gamma_{\sfU,k}$ in terms of its strong concavity and smoothness parameters (see Definitions \ref{def:submodularityRatio}--\ref{def:RSCRSM}). Proofs of lemmas and theorems omitted from this section can be found in Section \ref{sec:proofs}.

\begin{theorem}[RSC/RSM Implies Weak Submodularity]\label{thm:ratioBound}
	Define $f(\sfS)$ as in \eqref{eq:mlfunction}, 
	with a function $l(\cdot)$ that is ($m_{|\sfU| + k}$,$M_{|\sfU| + k}$)-(strongly concave, smooth) on 
	$\Omega_{|\sfU| + k}$ and $\tilde{M}_{|\sfU| + 1}$ smooth on $\tilde{\Omega}_{|\sfU| + 1}$. 
	Then the submodularity ratio $\gamma_{\sfU,k}$ is lower bounded by
	\begin{align}
	\gamma_{\sfU,k} \geq \frac{m_{|\sfU| + k}}{\tilde{M}_{|\sfU| + 1}} \geq \frac{m_{|\sfU| + k}}{M_{|\sfU| + k}} \;. \label{eq:submodBound}
	\end{align}
\end{theorem}

\begin{rem}
In the case of linear least-squares regression, $m$ and $M$ become sparse eigenvalues of the covariance matrix, \textit{i.e.} $m_{|\sfU| + k} = \lambda_{\min}(|\sfU| +k) \geq 0$ and $\tilde{M}_{|\sfU|+1} = \lambda_{\max}(1) = 1$. 
Thus Theorem \ref{thm:ratioBound} becomes $\gamma_{\sfU,k} \geq \lambda_{\min}(|\sfU| +k)$, \textit{i.e.} ``RIP implies weak submodularity'', consistent with Lemma 2.4 of \cite{Das2011}.
\end{rem}

\begin{rem}
Since $\nicefrac{m}{M} \leq 1$, this method cannot prove that the function is submodular (even on a restricted set of features). 
However, the guarantees in this section only require weak submodularity.
\end{rem}

\begin{rem}
Theorem \ref{thm:ratioBound} has the following geometric interpretation: the submodularity ratio of $f(\sfS)$ is bounded in terms of the maximum curvature of $l(\cdot)$ over the domain $\tilde{\Omega}_{|\sfU|+1}$ and the minimum curvature of $l(\cdot)$ over the (larger) domain $\Omega_{|\sfU|+k}$. 
The upper-curvature bound effectively controls the maximum amount that each individual function coordinate can influence the function value. The lower-curvature bound provides a lower-bound on the improvement of adding all features at once. Hence, loosely speaking the submodular ratio bound will be the ratio of both of these quantities. 
This intuition is used more formally in the proof (Section \ref{sec:proofs}), where for a $k$-sparse set $\sfS$ and $j \in \sfS$, $l(\bbeta^{(\sfU)})$ is perturbed 
by scaled projections of $\nabla l(\bbeta^{(\sfU)})$ onto 
$\bbeta^{(\sfU \cup \sfS)}_j$ and $\vece_{\sfS}$,
respectively.
\end{rem}

Theorem \ref{thm:ratioBound} allows us to generalize several results of \cite{Das2011}, starting with the following lemma:

\begin{lem}\label{lem:squeeze}
	Let $1 \leq k \leq n$.
	\begin{align*}
	f([k])  &\geq \max\left\{\frac{1}{k} , \frac{m_1}{4M_k}\left(3 + \frac{m_1}{M_1}\right) \right\} \sum_{j=1}^k f(j) \\
	&\geq \max\left\{\frac{1}{k} , \frac{m_k}{4M_k}\left(3 + \frac{m_k}{M_k}\right) \right\} \sum_{j=1}^k f(j).
	\end{align*}
\end{lem}

Now we present our first performance guarantee for feature selection.
\begin{theorem}[Oblivious Algorithm Guarantee]\label{thm:obliviousGuarantee}
	Define $f(\sfS)$ as in \eqref{eq:mlfunction}, with a function  $l(\cdot)$ that is $M_k$-smooth and $m_k$-strongly concave on 
	$\Omega_k$.
	Let $f^{OBL}$ be the value at the set selected by the Oblivious algorithm, and let $f^{OPT}$ be the optimal value over all sets of size $k$. Then 
	\begin{align*}
	f^{OBL} &\geq \max\left\{ \frac{m_k}{kM_1},\frac{m_km_1}{4M_kM_1}\left(3 + \frac{m_1}{M_1}\right) \right\}  f^{OPT} \\
	&\geq \max\left\{ \frac{m_k}{kM_k},\frac{3m_k^2}{4M_k^2}, \frac{m_k^3}{M_k^3} \right\}  f^{OPT} .
	\end{align*}
\end{theorem}

\begin{rem}
When the function is modular, \textit{i.e.} $m_{\Omega}=M_{\Omega}$ for all $\Omega$, then $f^{OBL} = f^{OPT}$ and the bound in Theorem \ref{thm:obliviousGuarantee} holds with equality.
\end{rem}

Next, we prove a stronger, constant factor approximation guarantee for the greedy, Forward Stepwise algorithm. 
\begin{theorem}[Forward Stepwise Algorithm Guarantee]\label{thm:greedyGuarantee}
Define $f(\sfS)$ as in \eqref{eq:mlfunction}, with a function that is $M$-smooth and $m$-strongly concave on 
$\Omega_{2k}$. 
Let $\sfS_k^{G}$ be the set selected by the FS algorithm and 
$\sfS^*$ 
be the optimal set of size $k$ corresponding to values $f^{FS}$ and $f^{OPT}$. Then 
\begin{align}
f^{FS} &\geq \left(1 - e^{-\gamma_{\sfS_k^G,k}}\right) f^{OPT} \geq \left(1 - e^{-\nicefrac{m}{M}}\right)  f^{OPT} . 
\label{eq:greedyBound}
\end{align}
\end{theorem}
\begin{rem}
This constant factor bound can be improved by running the Forward Stepwise algorithm for $r > k$ steps. The proof of Theorem \ref{thm:greedyGuarantee} generalizes to compare performance of $r$ greedy iterations to the optimal $k$-subset of features. This generalized bound does not necessarily approach $1$ as $r \rightarrow \infty$, however, since $\gamma_{\sfS_r^G,k}$ is a decreasing function of $r$.
\end{rem}
\begin{cor}
\label{cor:pastk}
Let $f^{FS+}$ denote the solution obtained after $r$ iterations of the Forward Stepwise algorithm, and let $f^{OPT}$ be the objective at the optimal $k$-subset of features. Let $\gamma = \gamma_{\sfS_r^G,k}$ be the submodularity ratio associated with the output of $f^{FS+}$ and $k$.
Then
\begin{align*}
f^{FS+} \geq (1 - e^{-\gamma (\nicefrac{r}{k})}) f^{OPT}.
\end{align*}
In particular, setting $r = ck$ corresponds to a $(1 - e^{-c \gamma})$-approximation, and setting $r = k \log n$ corresponds to a $(1 - n^{-\gamma})$-approximation.
\end{cor}

Corollary~\ref{cor:pastk} is useful when $\gamma$ can be bounded on larger support sets. We next present approximation guarantees when $\gamma$ can only be bounded on smaller support sets. 

\begin{theorem}\label{thm:greedyGuaranteeExtra}
Define $f(\sfS)$ as in \eqref{eq:mlfunction}, with a function $l(\cdot)$ that is $m^{\prime}$-strongly concave on $\Omega_k$ and $M^{\prime}$-smooth on $\tilde{\Omega}_k$. Let $\sfS_k^{G}$ be the set of features selected by the Forward Stepwise algorithm and $\sfS_k$ be the optimal feature set on $k$ variables corresponding to values $f^{G}$ and $f^{OPT}$. Then 
\begin{align*}
f^{FS} &\geq \Theta\left(2^{\nicefrac{-M^\prime}{m^\prime}}\right) \left(1 - e^{-\nicefrac{m^\prime}{M^\prime}}\right)  f^{OPT} .
\end{align*}
\end{theorem}

\begin{rem}
We note that our bounds are loose for certain special cases like modular functions and linear regression. These require making use of additional tools and specific properties of the function and data at hand (see~\cite{Das2011}). 
\end{rem}

Orthogonal Matching Pursuit is more computationally efficient than forward stepwise regression, since step $i$ only fits one regression instead of $p-i$. 
Thus we have a weaker guarantee than Theorem \ref{thm:greedyGuarantee}.
Similar to Corollary~\ref{cor:pastk}, this result generalizes to running OMP for $r > k$ iterations.

\begin{theorem}[OMP Algorithm Guarantee]\label{thm:ompGuarantee}
Define $f(S)$ as in \eqref{eq:mlfunction}, with a log-likelihood function that is 
($M$,$m$)-(smooth, strongly concave) on $\Omega_{2k}$.
Let $f^{OMP}$ be the value at the set of features selected by the OMP algorithm and $f^{OPT}$ be the optimal value over all sets of size $k$. Then
\begin{align*}
f^{OMP} &\geq \left(1 - e^{-\nicefrac{m}{{M}}}\right) f^{OPT} . 
\end{align*}
\end{theorem}

\begin{cor}
	\label{cor:pastkOMP}
	Let $f^{P+}$ denote the solution obtained after $r$ iterations of the OMP algorithm, and let $f^{OPT}$ be the objective at the optimal $k$-subset of features. Let $\alpha = (\nicefrac{m}{M})$ be the ratio associated with the output of $f^{P+}$ and $k$.
	Then
	\begin{align*}
	f^{P+} \geq (1 - e^{-\alpha (\nicefrac{r}{k})}) f^{OPT}.
	\end{align*}
	In particular, setting $r = ck$ corresponds to a $(1 - e^{-c \alpha})$-approximation, and setting $r = k \log n$ corresponds to a $(1 - n^{-\alpha})$-approximation.
\end{cor}

\begin{rem}\label{rem:ompImprove}
Theorem \ref{thm:ompGuarantee} and Corollary \ref{cor:pastkOMP} improve on the approximation guarantee of \cite{Das2011} by a factor of $\gamma$ in the exponent. Previous work obtained the approximation factor $1 - e^{-\gamma \lambda_{\min}(2k)}$, whereas the proof of Theorem \ref{thm:ompGuarantee} establishes $1 - e^{- \lambda_{\min}(2k)}$. Therefore we obtain a better exponent for linear regression and also generalize to any likelihood function. Theorem \ref{thm:greedyGuarantee} also gives intuition on when the performance of OMP will differ from that of Forward Selection, \textit{i.e.} when the inequality \eqref{eq:greedyBound} is loose. 
\end{rem}

\section{Statistical Recovery Guarantees}
\label{sec:stats}
Understanding optimization guarantees are useful, but do not clearly translate to bounds on parameter recovery. 
Below we present a general theorem that allows us to derive parameter bounds. When combined with Section \ref{sec:derivation}, it produces recovery guarantees for greedy algorithms as special cases.
\begin{theorem}[Parameter Recovery Guarantees]\label{thm:paramrecovery}
  Suppose that after $r$ iterations to approximate a function evaluated at a set $\sfS_s^*$ of cardinality $s$, we have the guarantee that
  \begin{equation*}
    f(\sfS_r) \geq C_{s,r} f(\sfS_s^*) .
  \end{equation*}
  Recall that $f(\sfS_r) = \max_{\text{supp}(\bbeta) \subset \sfS_r} l(\bbeta) - l(\bold0)$. Let $\thetagreedy^r$ be the solution to the optimization problem and consider any arbitrary $s$-sparse vector $\param^s$ with support on $\sfS_s^*$. Then, under $m_{s+r}$ RSC on $\Omega_{s+r}$ we have that
  \begin{equation*}
    \|\thetagreedy^r - \param^s\|_2^2 \leq \frac{4}{m_{s+r}^2} \|\nabla l(\param^s)\|^2_{2,(s+r)} + \frac{4}{m_{s+r}} (1-C_{s,r}) [ l(\param^s)-l(\bold0)] \; .
  \end{equation*}
\end{theorem}

For the remainder of this section, we consider several cases of Theorem \ref{thm:paramrecovery} and compare to results from previous work.
\subsection{Forward Selection with Linear Regression Model}
First we will consider a special case of Algorithm \ref{alg:greedy} for linear regression where the rows of the design matrices are $N(0,\Sigma)$ for a covariance matrix of the form $\Sigma = \matI + \mathbf{1} \mathbf{1}^T$. Further, we assume the model
  \begin{equation*}
    \vecy = \matX \param^* + \vecw \; ,
  \end{equation*}
  where $\|\param^*\|_2 \leq 1$ and is $s$-sparse, the rows of $\matX \in \mathbb{R}^{n \times p}$ are $N(0,\Sigma)$, and $w_i \sim N(0,\sigma^2)$ are i.i.d. We also take $l(\param) = \nicefrac{1}{n} \|\matX \param - \vecy\|_2^2$.
\begin{cor}\label{cor:paramLinear}
Given the above setup, if $(s+r)\sigma^2 \log p = o(n)$ 
and $r = \Omega(s \log n)$, then the parameter error goes to zero with high probability as $n \rightarrow \infty$.
\end{cor}

\subsection{Orthogonal Matching Pursuit with Linear Regression Model}
Next, we consider the results of Zhang, which provides parameter recovery bounds in the case of OMP (Algorithm~\ref{alg:omp}). The simplest comparison is to contrast our results with Corollary 2.2 of~\cite{Zhang11}. Consider the linear regression model above with an original $s$-sparse vector, $r$ iterations of the algorithm, and a spiked identity covariance model, $\Sigma = (1-a)\matI + a\mathbf{1} \mathbf{1}^T$. 
\begin{prop}
While Theorem \ref{thm:paramrecovery} holds for any $a$, Zhang~\cite{Zhang11} requires that $a$ does not exceed $\frac{1}{s+1}$.
\end{prop}

\begin{proof}
Zhang requires the RIP condition to hold, namely $M_{s} \leq 2 m_{s+r}$.
We know that the difference between means of $2\lambda_{\min}(s+r)$ and $\lambda_{\max}(s)$ is $\Delta = 1 - a - as$. Since $\nicefrac{\Delta}{2} \leq \nicefrac{\mu}{2}$ in both cases and $\chi^2$ variables concentrate within constant factors of their means, we have $M_s \leq \nicefrac{3}{2}(1 - a) + \nicefrac{as}{2} \leq 2m_{s+r}.$
However, $\Delta > 0 \Leftrightarrow s \leq \nicefrac{1}{a} - 1$. Rearranging, we have $a \leq \frac{1}{s+1}$. Thus, as has been noted in prior work, the RIP condition will not hold for the spiked model in settings where $a$ is much larger that $\frac{1}{s+1}$. 
\end{proof}

Nevertheless, we can still proceed and assume that the RIP condition is not required. In that case, the bound established in \cite{Zhang11} shows
\begin{equation*}
\|\thetagreedy^r - \param^* \|_2^2 \leq 24 M_{s+r} \|\matX  \param^* - \vecy \|_2^2/m_{s+r}^2 \;.
\end{equation*}
When $M_{s+r}$ and $m_{s+r}$ are of the same order, then this result is better than ours by log factors. However, when we consider a case like the spiked covariance model, then our results are better by a factor of $s$ in terms of statistical accuracy, but worse by a factor of $\log n$ with respect to sample complexity. 

\subsection{Orthogonal Matching Pursuit with Logistic Regression Model}
	Finally, consider our bounds for OMP (Algorithm~\ref{alg:omp}) in the case of logistic regression. Applying our approximation guarantees in Theorem~\ref{thm:ompGuarantee} matches the bound given by Theorem 2 of~\cite{Lozano2011} up to constant factors. However, their guarantee for parameter recovery requires a condition that is only known to be satisfied under incoherence assumptions. Our Theorem \ref{thm:paramrecovery} holds more generally. Their conditions on exact recovery are incomparable with our statistical error bounds.

\section{Theorem Proofs}
\label{sec:proofs}
\subsection{Proof of Theorem \ref{thm:ratioBound}}
\begin{proof}
We proceed by upper bounding the denominator and lower bounding the numerator of \eqref{eq:submodDef}. Let $\overline{k} = |\sfL|+k$.
First, we apply Definition \ref{def:RSCRSM} with $\bx = \bbeta^{(\sfL)}$ and $\by = \bbeta^{(\sfL \cup \sfS)}$,
\begin{align}
\frac{m_{\overline{k}}}{2} \pnorm{\bbeta^{(\sfL \cup \sfS)} - \bbeta^{(\sfL)}}{2}^2 &\leq  l(\bbeta^{(\sfL)}) - l(\bbeta^{(\sfL \cup \sfS)}) + \langle \nabla l(\bbeta^{(\sfL)}) , \bbeta^{(\sfL \cup \sfS)} - \bbeta^{(\sfL)} \rangle  .
\end{align}
Rearranging and noting that $l(\cdot)$ is monotone for increasing supports,
\begin{align*}
0 \leq l(\bbeta^{(\sfL \cup \sfS)}) - l(\bbeta^{(\sfL)})  &\leq \langle \nabla l(\bbeta^{(\sfL)}) , \bbeta^{(\sfL \cup \sfS)} - \bbeta^{(\sfL)} \rangle  - \frac{m_{\overline{k}}}{2} \pnorm{\bbeta^{(\sfL \cup \sfS)} - \bbeta^{(\sfL)}}{2}^2 \\
&\leq \max_{\vecv : \vecv_{(\sfL \cup \sfS)^c} = 0}   \langle \nabla l(\bbeta^{(\sfL)}) , \vecv - \bbeta^{(\sfL)} \rangle  - \frac{m_{\overline{k}}}{2} \pnorm{\vecv - \bbeta^{(\sfL)}}{2}^2 .
\end{align*}
Setting $\vecv  =  \bbeta^{(\sfL)} + \nicefrac{1}{m_{\overline{k}}}\nabla l(\bbeta^{(\sfL)})_{\sfS}$, we have
\begin{align}
0 \leq l(\bbeta^{(\sfL \cup \sfS)}) - l(\bbeta^{(\sfL)})  &\leq \frac{1}{2m_{\overline{k}}}\pnorm{\nabla l(\bbeta^{(\sfL)})_\sfS}{2}^2 . \label{eq:boundDenom2}
\end{align}

Next, consider a single coordinate $j \in \sfS$. The function at $\bbeta^{(\sfL \cup \{j\})}$ is larger than the function at any other $\bbeta$ on the same support. In particular $l(\bbeta^{(\sfL \cup \{j\})}) \geq l(\by_j )$, where $\by_j := \bbeta^{(\sfL)} + \alpha_j \bbeta^{(\sfL \cup \sfS)}_j$ for some scalar $\alpha_i$. Noting that $(\vecx = \bbeta^{(\sfL)},\by = \by_j) \in \tilde{\Omega}_{|\sfL|+1}$ and applying Definition \ref{def:RSCRSM},
\begin{align*}
l(\bbeta^{(\sfL \cup \{j\})}) - l(\bbeta^{(\sfL)}) &\geq l(\bbeta^{(\sfL)} + \alpha_j \bbeta^{(\sfL \cup \sfS)}_j  ) - l(\bbeta^{(\sfL)}) \\
&\geq \langle \nabla l(\bbeta^{(\sfL)}) , \alpha_j \bbeta^{(\sfL \cup \sfS)}_{j} \rangle - \frac{\tilde{M}_{|\sfL|+1}}{2} |\alpha_j \bbeta^{(\sfL \cup \sfS)}_{j}|^2 .
\end{align*}
Summing over all $j \in \sfS$ and setting
\begin{align*}
\alpha_j = \frac{\langle \nabla l(\bbeta^{(\sfL)}) ,  \bbeta^{(\sfL \cup \sfS)}_{j} \rangle}{\tilde{M}_{|\sfL|+1} | \bbeta^{(\sfL \cup \sfS)}_{j}|^2} ,
\end{align*}
we have
\begin{align*}
l(\bbeta^{(\sfL \cup \{j\})}) - l(\bbeta^{(\sfL)}) &\geq \frac{ \left( \langle \nabla l(\bbeta^{(\sfL)}) ,  \bbeta^{(\sfL \cup \sfS)}_{j} \rangle \right) ^2}{2\tilde{M}_{|\sfL|+1} | \bbeta^{(\sfL \cup \sfS)}_{j}|^2} \nonumber \\
\Rightarrow \sum_{j \in S} l(\bbeta^{(\sfL \cup \{j\})}) - l(\bbeta^{(\sfL)}) &\geq \frac{1}{2\tilde{M}_{|\sfL|+1}} \sum_{j \in S} \left( \nabla l(\bbeta^{(\sfL)})_j \right)^2 = \frac{1}{2\tilde{M}_{|\sfL|+1}} \pnorm{\nabla l(\bbeta^{(\sfL)})_\sfS}{2}^2 .
\end{align*}
Substituting the above line and \eqref{eq:boundDenom2} into \eqref{eq:submodDef}, the result follows from taking the minimum over all sets 
$\sfL, \sfS$.
\end{proof}

\subsection{Proof of Lemma \ref{lem:squeeze}}
\begin{proof}
Let $\sfS = [k]$. Since $f(\cdot)$ is monotone, $f(j) \leq f(\sfS)$ for $j \in \sfS$. Summing over all $j \in \sfS$ and dividing by $k$ yields the first part of the inequality. 
The rest of the proof requires combining several applications of Definition \ref{def:RSCRSM} to the underlying likelihood function $l$ for carefully chosen $\vecx,\vecy$.
Define a $k$-sparse 
$\overline{\bbeta}$ by $\overline{\bbeta}_j = \alpha_j \bbeta_j^{(j)}, j \in \sfS$ for some positive scalar $a_j$ and $0$ elsewhere.
First we apply Definition \ref{def:RSCRSM} with $\vecx = \mathbf{0}$ and $\vecy = \overline{\bbeta}$. This implies
\begin{align}
l(\overline{\bbeta}) \geq \langle \nabla l(\mathbf{0}),\overline{\bbeta} \rangle - \frac{M_k}{2}\sum_{j \in \sfS} |\alpha_j \bbeta_j^{(j)}| ^2 . \label{eq:tmpA}
\end{align}
Next, applying the same definition $k$ times with $\vecx = \mathbf{0}$ and $\vecy = \bbeta^{(j)}$ and summing over $j \in \sfS$,
\begin{align}
\langle \nabla l(\mathbf{0}),\alpha_j \bbeta^{(j)} \rangle &\geq \alpha_j \left( l(\bbeta^{(j)}) + \frac{m_1}{2} |\bbeta^{(j)}_j|^2 \right) \nonumber \\
\Rightarrow \langle \nabla l(\mathbf{0}),\overline{\bbeta} \rangle &\geq \sum_{j \in \sfS} \alpha_j l(\bbeta^{(j)}) +  \alpha_j  \frac{m_1}{2} |\bbeta^{(j)}_j|^2    . \label{eq:tmpB}
\end{align}
Combining \eqref{eq:tmpA} with \eqref{eq:tmpB}, and setting 
\begin{align*}
\alpha_j = \frac{m_1}{2M_k} + \frac{l(\bbeta^{(j)})}{M_k|\bbeta_j^{(j)}|^2}  ,
\end{align*}
we have
\begin{align}
l(\overline{\bbeta}) \geq \sum_{j \in \sfS} \frac{m_1}{2M_k} l(\bbeta^{(j)}) + \frac{m_1^2}{8M_k} |\bbeta^{(j)}|^2 + \frac{(l(\bbeta^{(j)}))^2}{2M_k|\bbeta_j^{(j)}|^2} . \label{eq:boundbbar}
\end{align}
Now applying Definition \ref{def:RSCRSM} with $\vecx = \bbeta^{(j)}$ and $\vecy = \mathbf{0}$,
\begin{align}
\frac{M_1}{2} |\bbeta_j^{(j)}|^2 &\geq l(\bbeta^{(j)}) \geq \frac{m_1}{2} |\bbeta_j^{(j)}|^2 . \label{eq:boundquad}
\end{align}
Combining \eqref{eq:boundbbar} and \eqref{eq:boundquad}, we have
\begin{align*}
l(\overline{\bbeta}) &\geq \sum_{j \in \sfS} \frac{m_1}{2M_k} l(\bbeta^{(j)}) + \frac{m_1^2}{4M_kM_1}l(\bbeta^{(j)}) + \frac{m_1}{4M_k} l(\bbeta^{(j)}) \\ &= \sum_{j \in \sfS} \left( \frac{3m_1}{4M_k} + \frac{m_1^2}{4M_kM_1} \right) l(\bbeta^{(j)}).
\end{align*}
Since
$l(\bbeta^{(\sfS)})$ optimizes $l$ over all vectors with support in $\sfS$,
\begin{align*}
f(\sfS) = l(\bbeta^{(\sfS)}) \geq l(\overline{\bbeta}) &\geq \frac{m_1}{4M_k}\left(3+\frac{m_1}{M_1}\right)  \sum_{j \in \sfS} l(\bbeta^{(j)}) \\ &= \frac{m_1}{4M_k}\left(3+\frac{m_1}{M_1}\right)\sum_{j=1}^k f(j) . \qedhere
\end{align*}
\end{proof}

\subsection{Proof of Theorem \ref{thm:obliviousGuarantee}}
\begin{proof}
	Let $\sfS$ be the set of size $k$ selected by the Oblivious algorithm and $\sfS^*$ be the optimal set of size $k$ corresponding to values $f^{OBL}$ and $f^{OPT}$. 
	By definition, $\sum_{j \in \sfS} f(j) \geq \sum_{j \in \sfS^*} f(j)$. Letting $C = \max \{\nicefrac{1}{k} , 
	\nicefrac{3m}{4M} + \nicefrac{m^2}{4M^2}) 
	\}$ and combining Lemma \ref{lem:squeeze} with Theorem \ref{thm:ratioBound},
	\begin{align*}
	f^{OBL} = f(\sfS) &\geq C \sum_{j \in \sfS}   f(j) \\ &\geq C \sum_{j \in \sfS^*}   f(j) \geq C \gamma_{\emptyset,k} f(\sfS^*) \geq C \left(\frac{m_k}{M_1}\right) f(\sfS^*) \\ &= C \left(\frac{m_k}{M_1}\right) f^{OPT}. \qedhere
		\end{align*} 
\end{proof}

\subsection{Proof of Theorem \ref{thm:greedyGuarantee}}
\begin{proof}
Let $l(\cdot)$ be the log-likelihood function and let $\sfS_i^G$ be the set selected by the Forward Stepwise algorithm at iteration $i$. Define $A(i)$ as the incremental greedy gain $f(\sfS_{i}^G) - f(\sfS_{i-1}^G)$ with $A(0) = 0$. Denote the remainder set at iteration $i$ as $\sfS_i^R = \sfS^* \backslash \sfS_i^G$, 
and define $B(i) = f(\sfS^*) - f(\sfS_{i}^G)$, the incremental gain from adding the optimal set. 
Lemma~\ref{lem:greedyIncrement} relates these two quantities.
\begin{lem}\label{lem:greedyIncrement}
At iteration $i$, the incremental gain from selecting the next greedy item is related to the incremental gain from adding the rest of the optimal set $\sfS^*$ by the following:
\begin{align*}
A(i+1) &\geq \frac{\gamma_{\sfS_i^G,k}}{k} B(i) \,.
\end{align*}
\end{lem}
\begin{proof}
Let $\sfS = \sfS_i^G$ be the set selected by the greedy algorithm at iteration $i$, $\sfS^*$ be the optimal feature set on $k$ variables, and $\sfS^R$ be the remainder set $\sfS^* \backslash \sfS$. $\sfS^R$ is a subset of the candidate variables available to the greedy algorithm at iteration $i+1$. Using Definition \ref{def:submodularityRatio} and the fact that $k \geq |\sfS^R|$,
\begin{align*}
k A(i+1) \geq |\sfS^R| A(i+1) &\geq |\sfS^R| \max_{j \in \sfS^R}  f(\sfS \cup j) - f(\sfS) \\
&\geq  \sum_{j \in \sfS^R} \left[ f(\sfS \cup j) - f(\sfS) \right]  \\ 
&\geq \gamma_{\sfS,|\sfS^R|} \left( f(\sfS \cup \sfS^R) - f(\sfS) \right) \geq \gamma_{\sfS,k} B(i) ,
\end{align*}
where the last inequality follows from the fact that $\sfS \cup \sfS^R \supseteq \sfS^*$.
\end{proof}

Given Theorem \ref{thm:ratioBound} and Lemma \ref{lem:greedyIncrement}, the rest of the proof follows the standard approximation bound for maximizing a normalized, monotone submodular function (refer to \cite{Nemhauser1978} 
or the survey \cite{Krause2014}). 
Next, observe that $A(i+1) = B(i) - B(i+1)$. Combining this with Lemma \ref{lem:greedyIncrement} and letting $C = \gamma_{\sfS_k^G,k}/k$, we have the following inequality:
\begin{align*}
B(i+1) &\leq \left(1 - C\right) B(i), 
\end{align*} 
which implies 
\begin{align*}
B(i) \leq \left(1 - C\right)^i B(0) ,
\end{align*}
for all iterations $1 \leq i \leq k$. Setting $i = k$ and substituting $B(k) = f^{OPT} - f^{FS}$ and $B(0) = f^{OPT}$,
\begin{align*}
f^{OPT} - f^{FS} &\leq \left( 1 - C \right)^k f^{OPT} \\ 
\Rightarrow f^{FS} &\geq f^{OPT} \left[1 -  \left(1 - C \right)^k \right]  \geq f^{OPT} \left(1 - e^{-\gamma_{\sfS_k^G,k}}\right).
\end{align*}
The claim follows from applying Theorem \ref{thm:ratioBound}.
\end{proof}

\subsection{Proof of Theorem \ref{thm:greedyGuaranteeExtra}}
\begin{proof}
First we prove the following lemma which bounds the ratio of the objective between optimal sets $\sfS_k$ and $\sfS_{k-1}$ in terms of their smoothness and convexity parameters. 
\begin{lem}\label{lem:optMult}
Let $\sfS_k$ be the optimal subset of size $k$, and let $m$ be the restricted strong concavity parameter on $\Omega_k$. Let $k'$ satisfy $\nicefrac{M^{\prime}}{m} < k' < k$, 
where $M^{\prime}$ is the restricted smoothness parameter of $l(\cdot)$ on $\tilde{\Omega}_k$.
Then for large enough $k$,
\begin{align*}
l(\bbeta^{(\sfS_{k^{\prime}})}) &\geq l(\bbeta^{(\sfS_{k})}) \Theta\left(\left(\frac{k'}{k}\right)^{\nicefrac{M^{\prime}}{m}}\right) 
\Rightarrow l(\bbeta^{(\sfS_{k/2})}) \geq l(\bbeta^{(\sfS_{k})}) \Theta\left(2^{\nicefrac{-M^{\prime}}{m}}\right) .
\end{align*}
\end{lem}
\begin{proof}
Let $j$ be the index that minimizes $|\bbeta^{(\sfS_k)}_j|^2$. By $M^{\prime}$-smoothness on $\tilde{\Omega}_k$ and the fact that the min is smaller than the average,
\begin{align*}
l(\bbeta^{(\sfS_{k-1})}) &\geq l(\bbeta^{(\sfS_{k})}_{\sfS_k \backslash \{j\}}) \\
&\geq l(\bbeta^{(\sfS_{k})}) + \langle \nabla l(\bbeta^{(\sfS_{k})}), \bbeta^{(\sfS_{k})}_{\sfS_k \backslash \{j\}} - \bbeta^{(\sfS_{k})} \rangle - \frac{M^{\prime}}{2}\pnorm{\bbeta^{(\sfS_{k})}_{\sfS_k \backslash \{j\}} - \bbeta^{(\sfS_{k})}}{2}^2 \\
&=  l(\bbeta^{(\sfS_{k})}) - \frac{M^{\prime}}{2}|\bbeta^{(\sfS_{k})}_j|^2 \\
\Rightarrow \frac{l(\bbeta^{(\sfS_{k-1})})}{l(\bbeta^{(\sfS_{k})})} &\geq 1 - \frac{M^{\prime}\pnorm{\bbeta^{(\sfS_{k})}}{2}^2}{2kl(\bbeta^{(\sfS_{k})})} \;.
\end{align*}
Assuming that $l(\bbeta^{(\emptyset)}) = 0$ and using 
$m$-strong concavity on $\Omega_k$,
\begin{align*}
l(\bbeta^{(\emptyset)}) - l(\bbeta^{(\sfS_{k})}) &\leq -\frac{m}{2} \pnorm{\bbeta^{(\sfS_{k})} - \bbeta^{(\emptyset)}}{2}^2 \Rightarrow -\frac{\pnorm{\bbeta^{(\sfS_{k})}}{2}^2}{l(\bbeta^{(\sfS_{k})})} \geq -\frac{2}{m} \\
\Rightarrow \frac{l(\bbeta^{(\sfS_{k-1})})}{l(\bbeta^{(\sfS_{k})})} &\geq 1 - \frac{M^{\prime}}{k m} \;.
\end{align*}
Then applying iteratively for $\nicefrac{M^{\prime}}{m}$ constant, $k$ large, and $\nicefrac{M^{\prime}}{m} < k' < k$, as in \cite{Das2011} we have
\begin{align*}
l(\bbeta^{(\sfS_{k^{\prime}})}) &\geq l(\bbeta^{(\sfS_{k})}) \prod_{j=k'+1}^k 1 - \frac{M^{\prime}}{j m} = l(\bbeta^{(\sfS_{k})}) \Theta\left(\left(\frac{k'}{k}\right)^{\nicefrac{M^{\prime}}{m}}\right) . \qedhere
\end{align*}
\end{proof}

Observe that the assumptions of Lemma \ref{lem:optMult} are satisfied. Combining with Theorem \ref{thm:greedyGuarantee},
\begin{align*}
l(\bbeta^{(\sfS_{k}^G)}) &\geq l(\bbeta^{(\sfS_{k/2}^G)}) \geq l(\bbeta^{(\sfS_{k/2})})\left(1 - e^{-\gamma_{\sfS_{k/2}^G,k/2}}\right) \\
&\geq l(\bbeta^{(\sfS_{k})}) \Theta\left(2^{\nicefrac{-M^{\prime}}{m}}\right) \left(1 - e^{-\gamma_{\sfS_{k/2}^G,k/2}}\right) \\
\Rightarrow l(\bbeta^{(\sfS_{k}^G)}) &\geq l(\bbeta^{(\sfS_{k})}) \Theta\left(2^{\nicefrac{-M^{\prime}}{m^{\prime}}}\right) \left(1 - e^{-\nicefrac{m^{\prime}}{M^{\prime}}}\right) . \qedhere
\end{align*}
\end{proof}

\subsection{Proof of Theorem \ref{thm:ompGuarantee}}
\begin{proof}
The key idea at each step $i$ is to lower bound the incremental gain from the index chosen by OMP. This is similar to the proof of Theorem \ref{thm:greedyGuarantee}, as well as \cite{Khanna2017matrix} in which a matrix completion objective is considered. Let $\sfS = \sfS_i^{P}$ be the set chosen by OMP up to iteration $i$. Given $\sfS$, let $v$ be the index that would be selected by running one additional step of OMP. Define $D(i+1) = f(\sfS_{i+1}^{P}) -f(\sfS) = l(\bbeta^{(\sfS \cup \{v\})}) - l(\bbeta^{(\sfS)})$, and define $\tilde{B}(i) = f(\sfS^{*}) -f(\sfS)$.
\begin{lem}\label{lem:ompIncrement}
	At iteration $i$, the incremental gain from selecting the next item via OMP is related to the incremental gain from adding the rest of the optimal set $\sfS^*$ by the following:
	\begin{align*}
		D(i+1) \geq \frac{m_{i+k}}{k\tilde{M}_{i+1}} \tilde{B}(i) .
	\end{align*}
\end{lem}
\begin{proof}
	We begin similar to the proof of Theorem \ref{thm:ratioBound}. Let $M=\tilde{M}_{i+1}$, $m=m_{i+k}$, and $\vece_v$ be the unit vector with one at coordinate $v$. By Definition \ref{def:RSCRSM} with $\vecx = \bbeta^{(\sfS)} $ and $\vecy = \bbeta^{(\sfS)} + \alpha \vece_v$ for any scalar $\alpha$,
	\begin{align*}
		D(i+1) \geq l(\vecy) - l(\vecx) &\geq \left\langle \nabla l(\bbeta^{(\sfS)}) , \alpha \vece_v \right\rangle - \frac{M}{2}\alpha^2 \\ 
		&= \alpha \pnorm{\nabla l(\bbeta^{(\sfS)})}{\infty} - \frac{M}{2}\alpha^2 \; ,
	\end{align*}
	since OMP chooses the coordinate which maximizes the gradient. 
	Substituting 
	\begin{align*}
	\alpha =  \frac{\pnorm{\nabla l(\bbeta^{(\sfS)})}{\infty}}{M}   \; , 
	\end{align*}
		we have 
	\begin{align*}
	D(i+1) \geq \frac{1}{2M}\pnorm{\nabla l(\bbeta^{(\sfS)})}{\infty}^2
	\end{align*}
	Let $\sfS^R = \sfS^* \backslash \sfS$. Since $|\sfS^R| \leq k$,
	\begin{align*}
	D(i+1) &\geq \frac{1}{2kM} \sum_{j \in \sfS^R} \left\langle \nabla l(\bbeta^{(\sfS)}) ,  \vece_j \right\rangle^2 = \frac{1}{2kM} \pnorm{\nabla l(\bbeta^{(\sfS)})_{\sfS^R}}{2}^2. 
	\end{align*}
	Substituting \eqref{eq:boundDenom2} into the above and noting that $\sfS \cup \sfS^R \supseteq \sfS^*$, we have
	\begin{align*}
	D(i+1) &\geq \frac{m}{kM} \left( l(\bbeta^{(\sfS \cup \sfS^R)}) - l(\bbeta^{(\sfS)}) \right) \geq  \frac{m}{kM} \tilde{B}(i) \; . \qedhere
	\end{align*}
\end{proof}
Given Lemma \ref{lem:ompIncrement}, the rest of the proof follows that of Theorem \ref{thm:greedyGuarantee}.
\end{proof}

\subsection{Proof of Theorem \ref{thm:paramrecovery}}
\begin{proof}
	Let $C = C_{s,r}$ and $\Delta = \thetagreedy^r - \param^s$, which is at most an ($s+r$)-sparse vector. Recall that by the definition of Restricted Strong Concavity on $\Omega_{s+r}$ we have
	\begin{equation}
	l(\thetagreedy^r) - l(\param^s) - \langle \nabla l (\param^s),\Delta \rangle \leq \frac{-m_{s+r}}{2} \|\Delta\|_2^2 .\label{eq:rscapply2}
	\end{equation}
	Furthermore, simple calculations show that
	\begin{equation}
	l(\thetagreedy^r)-l(\param^s) \geq (1-C) [ l(\bold0)  - l(\param^s)] .	\label{eq:diff1}	
	\end{equation}
	Subtracting $\langle \nabla l (\param^s),\Delta \rangle$ from both sides of \eqref{eq:diff1} we have
	\begin{equation*}
	\label{eq:doubt2}
	l(\thetagreedy^r)-l(\param^s) - \langle \nabla l (\param^s),\Delta \rangle \geq - \langle \nabla l (\param^s),\Delta \rangle + (1-C) [ l(\bold0) - l(\param^s)  ] \;.
	\end{equation*}
	Applying \eqref{eq:rscapply2} yields
	\begin{equation*}
	\frac{-m_{s+r}}{2} \|\Delta\|_2^2 \geq - \langle \nabla l (\param^s),\Delta \rangle + (1-C) [ l(\bold0) - l(\param^s)  ] \;.
	\end{equation*}
	Next, note that
	\begin{equation*}
	- \langle \nabla l(\param^s) , \Delta \rangle \geq - \|\nabla l(\param^s)\|_{2,s+r} \|\Delta\|_2 \;.
	\end{equation*}
	Thus,
	\begin{equation*}
	\frac{-m_{s+r}}{2} \|\Delta\|_2^2 \geq - \|\nabla l(\param^s)\|_{2,k} \|\Delta\|_2 + (1-C) [ l(\bold0) - l(\param^s)  ] \;.
	\end{equation*}
	Recalling that for any positive numbers $2a b \leq c a^2 + b^2/c$ and flipping the above inequality,
	\begin{equation*}
	\frac{m_{s+r}}{2} \|\Delta\|_2^2 \leq \frac{  \|\nabla l(\param^s)\|_{2,s+r}^2}{m_{s+r}} +  \frac{m_{s+r} \|\Delta\|_2^2}{4} + (1-C) [ l(\param^s) - l(\bold0)] \;.
	\end{equation*}
	Rearranging terms we have the final result.
\end{proof}

\subsection{Proof of Corollary \ref{cor:paramLinear}}
\begin{proof}
Using a result of \cite{Negahban2012journal}, we have that 
\begin{align*}
\|\nabla l(\param^s)\|_{2,s+r}^2 \leq (s+r) \pnorm{\nabla l(\param^s)}{\infty}^2 \leq \frac{(s+r) \sigma^2 \log p}{n} \; ,
\end{align*}
with probability at least $1 - \nicefrac{1}{p}$.

  The minimum eigenvalue of the matrix $\Sigma$ is $1$, while the maximum $s$-sparse eigenvalue behaves like $1+s$. Hence, an RIP type condition will not hold. However, in our setting, we simply require a bound on $M_1$. It can be shown using tail bounds for $\chi^2$-random variables that with high probability $M_1 \leq 4$. Letting $\rho(\Sigma)^2 = \max_i \Sigma_{ii}$, and using a result by Raskutti et. al.~\cite{Raskutti2010}, we have that for all $\vecv \in \bbR^p$,
  \begin{align*}
  \frac{\pnorm{\matX \vecv}{2}}{\sqrt{n}} &\geq \nicefrac{1}{4} \pnorm{\Sigma^{\nicefrac{1}{2}}\vecv}{2} - 9\rho(\Sigma)\sqrt{\frac{\log p}{n}} \pnorm{\vecv}{1} \\
  &\geq \left(\left(1 - \nicefrac{1}{c} \right)\frac{\lambda_{\min}(\Sigma)}{16}   + \left(1 - c\right)81\rho(\Sigma)^2\frac{\log p}{n} (s+r)\right)\pnorm{\vecv}{2}^2  \\
  \Rightarrow m_{s+r} &\geq \min_{\substack{\vecv : \pnorm{\vecv}{2}=1, \\\pnorm{\vecv}{0}\leq s+r}} \frac{\pnorm{\matX \vecv}{2}^2}{n} \geq \frac{1}{32} - \frac{162(s+r)\log p}{n} \; ,
    \end{align*}
   with high probability. Therefore,
    $\gamma \geq  \frac{1}{128} - \frac{81(s+r) \log p}{2n}$ and
with probability at least $1 - p^{-\Omega(1)} - e^{-\Omega(n)}$,
  \begin{equation*}
    \|\thetagreedy^r - \param^*\|_2^2 \leq \frac{4}{m_{s+r}^2} \frac{(s+r) \sigma^2 \log p}{n} + \frac{8(s+1)}{m_{s+r}} (1-C_{s,r}) \;,
  \end{equation*}
  where we have used the fact that $l(\param^*) - l(\bold0) \leq \lambda_{\max}(\hat{\Sigma}_s) \leq 2(s+1)$ with high probability. Note that using arguments from~\cite{Negahban2012preprint,loh2015} we can apply the above results to the setting of generalized linear models.

Now let $(s+r)\sigma^2 \log p = o(n)$
and $r = \Omega(s \log n)$. Combined with Corollary~\ref{cor:pastk}, this implies that
$\|\thetagreedy^r - \param^*\|_2^2 =  n^{-\Omega(1)}$ with probability $1 - p^{-\Omega(1)} - e^{-\Omega(n)}$.
\end{proof}

\section{Experiments} 
Next we evaluate the performance of our greedy algorithms with feature selection experiments on simulated and real-world datasets. A bias term $\beta_0$ is added to the regression by augmenting the design matrix with a column of ones. 
\paragraph{The Data}
A synthetic experiment was conducted as follows: first each row of a $600 \times 200$ design matrix $\matX$ is generated independently according to a first order AR process ($\alpha = 0.3$ and noise variance $\sigma^2=5$). This ensures that the features are heavily correlated with each other. Bernoulli $\pm 1$ (\textit{i.e.}, Rademacher) random variables are placed on $50$ random indices to form the true support $\thetabar^k$, and scaled such that  $\pnorm{\bbeta}{2}^2=5$. Then responses $\vecy$ are computed via a logistic model. We also conduct an experiment on a subset of the RCV1 Binary text classification dataset \cite{rcv1Dataset}. $10,\!000$ training and test samples are used in $47,\!236$ dimensions. Since there is no ground truth, a logistic regression is fit using a subset of at most $700$ features.

\paragraph{Algorithms and Metrics}
The Oblivious, Forward Stepwise (FS), and OMP algorithms were implemented using a logistic log-likelihood function
given $\matX$ and $\vecy$ (see Appendix~\ref{sec:glmAppendix}).
We implemented $3$ additional algorithms. \textit{Lasso} fits a logistic regression model with $\ell_1$ regularization. \textit{Lasso-Pipeline} recovers the sparse support using Lasso and then fits regression coefficients on this support with a separate, unregularized model. The regularization parameter was swept to achieve outputs with varying sparsity levels. 
\textit{Forward Backward} (FoBa) \cite{foba} first runs FS at each step and then drops any features if doing so would decrease the objective by less than half the latest marginal gain.

Our main metric for each algorithm is the normalized objective function $\functionvector(\thetagreedy^s)$  - $\functionvector(\bf 0)$ for the output sparsity $s \in \{1,\ldots,70\}$. 
We also compare the sets $\supp(\thetagreedy^s)$ and $\supp(\thetabar^k)$ using area under ROC and percent of true support recovered.
Finally, we measure generalization accuracy by drawing additional observations $(\feature,y_i)$ from the same distribution as the training data.
\paragraph{Results}

Figure \ref{fig:arLogistic} shows the results of our synthetic experiment averaged over $20$ runs. For all metrics, Oblivious performs worse than OMP which is slightly worse than FS and FoBa. This matches intuition and the series of bounds in Section \ref{sec:derivation}. 
We also see that the Lasso-Pipeline performs noticeably worse than all algorithms except Oblivious and Lasso. This suggests that 
greedy feature selection degrades more gracefully than Lasso in the case of correlated features.

Figure \ref{fig:rcv1Logistic} shows similar results for the high-dimensional RCV1 Binary dataset. Due to their large running time complexity, FS and FoBa were omitted. While all algorithms have roughly the same generalization accuracy using $300$ features, OMP has the largest log-likelihood.

\begin{figure}[ht]
\centering
\begin{subfigure}{0.65\linewidth}
\includegraphics[width=\linewidth]{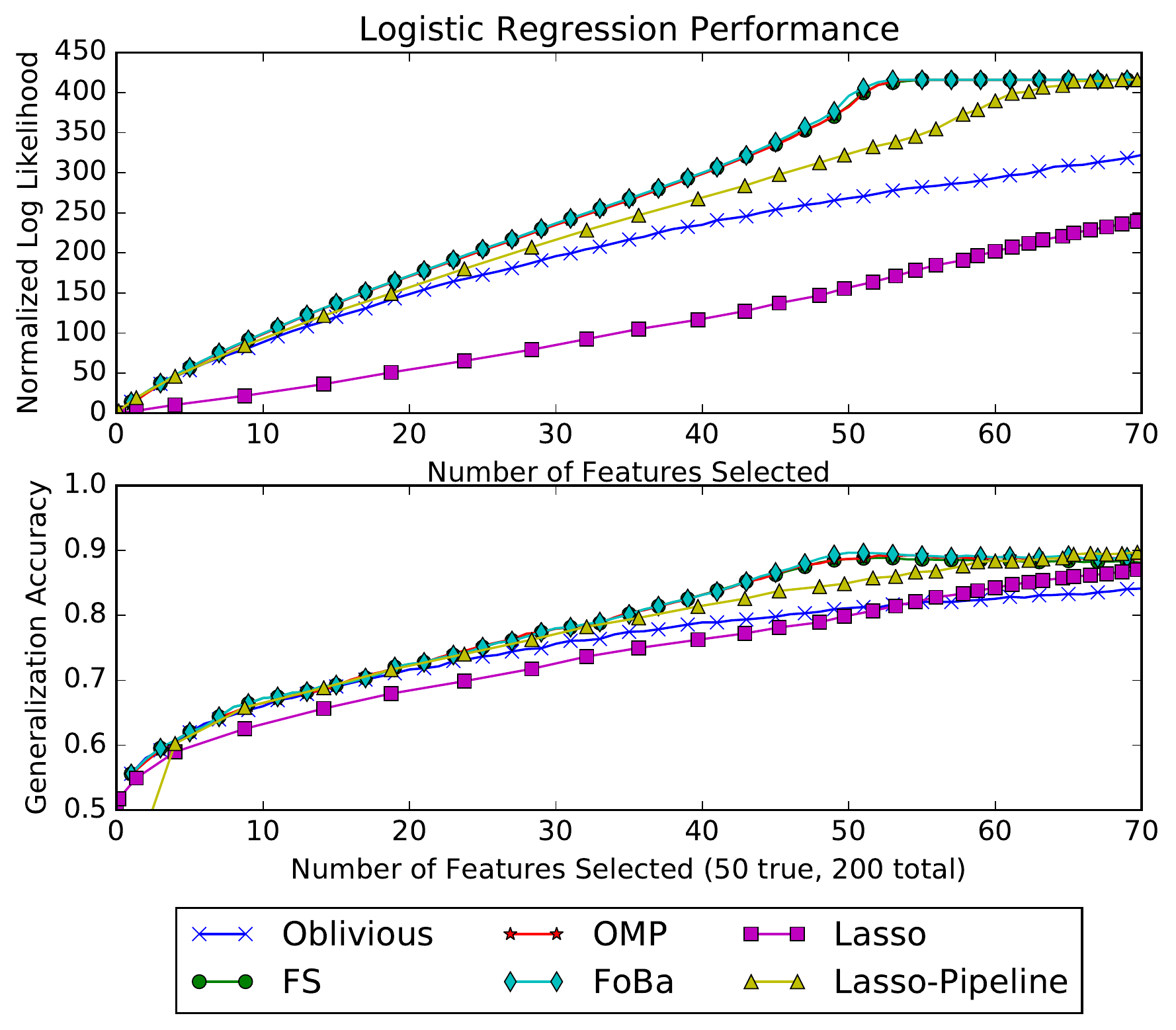}
\caption{}
\label{fig:arLL}
\end{subfigure} \\
\begin{subfigure}{0.65\linewidth}
\includegraphics[width=\linewidth]{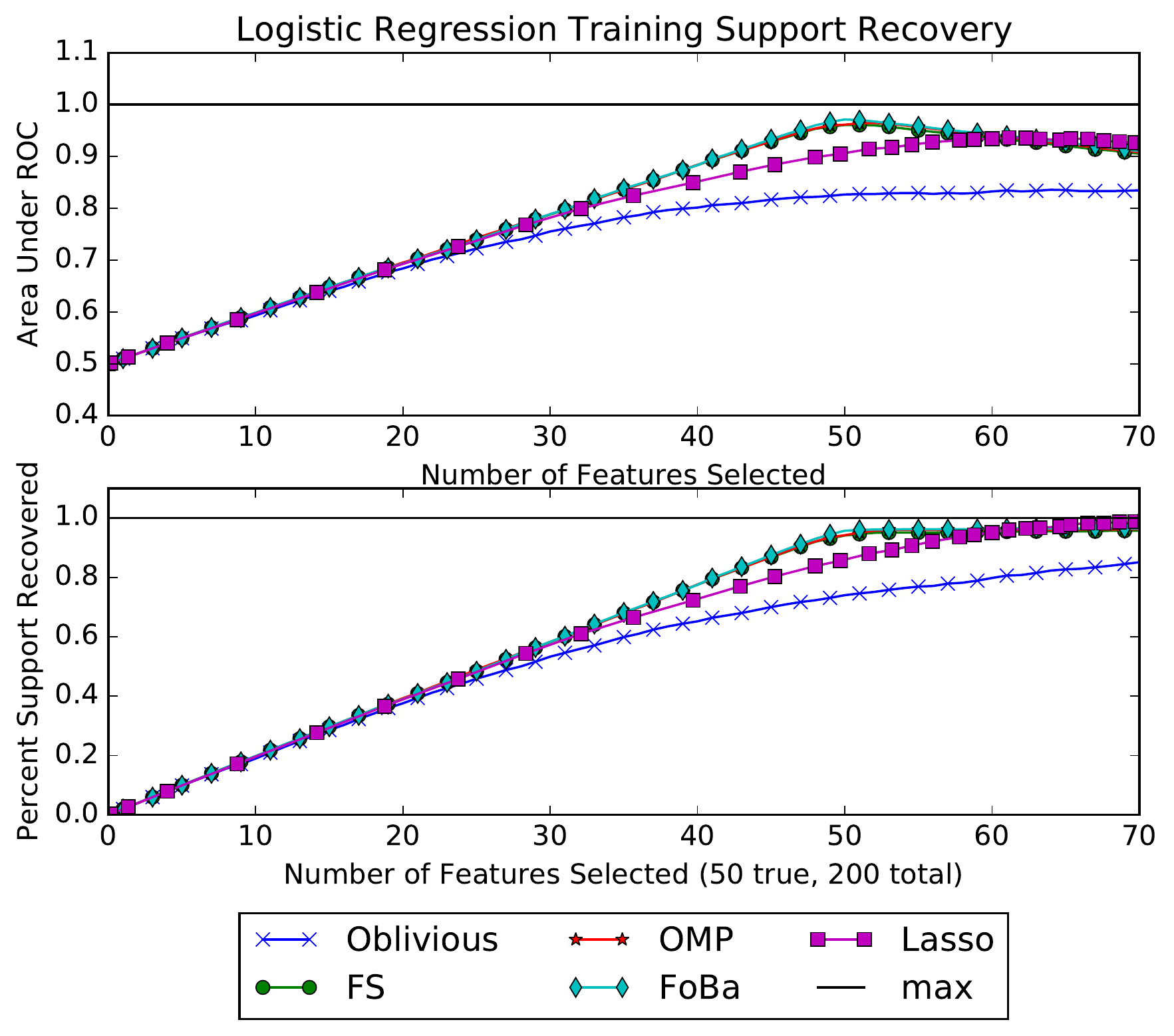}
\caption{}
\label{fig:arSupport}
\end{subfigure}
\caption{\label{fig:arLogistic}Synthetic Dataset - $\alpha = 0.3$, $n=600$ training and test samples, $p=200$ dimensions with true support on $50$ features, averaged over $20$ runs. 
	(a) The greedy algorithms perform better than Lasso and Oblivious algorithms, but beyond $50$ steps they overfit to noise in the training data. While Lasso outperforms Oblivious in support recovery (b), its regression suffers from regularization bias.
	}
\end{figure}

\begin{figure}[ht]
	\centering
		\includegraphics[width=0.65\linewidth,trim=0 1.1in 0 0, clip=true]{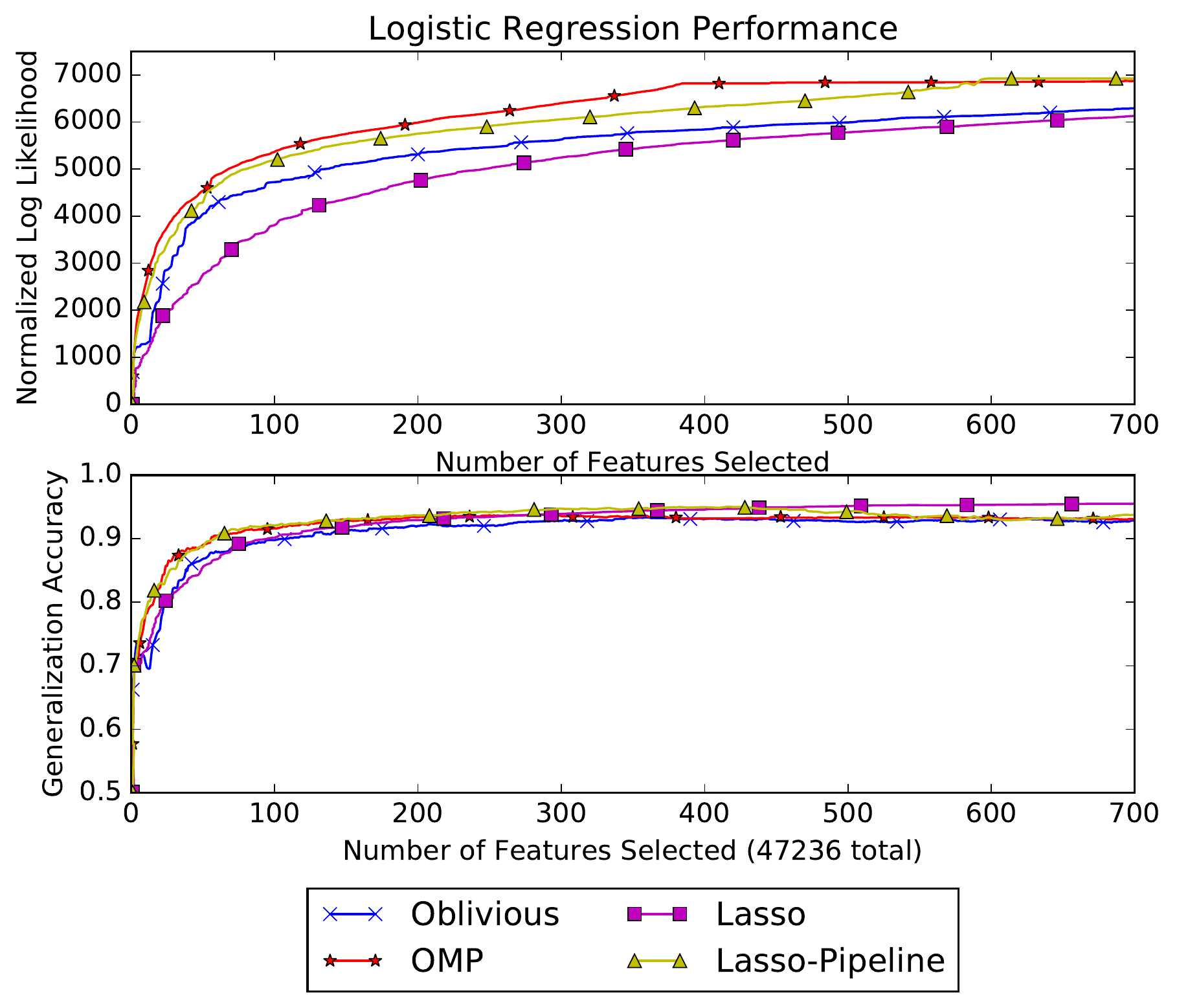}
	\caption{\label{fig:rcv1Logistic}RCV1 Binary Dataset - $n=10,\!000$, $p=47,\!236$. OMP outperforms Lasso-Pipeline.
		 }
\end{figure}

\section{Conclusions}
We have extended the results of \cite{Das2011} and shown that functions satisfying RSC also satisfy a relaxed form of submodularity that can be used to analyze the performance of greedy algorithms compared to the best sparse solution. Experimental results confirm that greedy feature selection outperforms regularized approaches in a nonlinear regression model. Directions for future work include similar analysis for other greedy algorithms that incorporate group sparsity \cite{Negahban2012journal} or thresholding, and applications beyond sparse regression. Bounds on dictionary selection (analogous to those in \cite{Das2011}) 
also apply to general likelihood functions satisfying RSC and RSM.

\clearpage 

\appendix
\section{Motivating Example (Linear Regression)}
To show the impact of submodularity, we construct a linear regression example. Even in $p=3$ dimensions, the greedy forward selection algorithm's output can be arbitrarily off from the optimal $R^2$. Consider the following variables:

\begin{align*}
\vecy &= \begin{bmatrix} 1 & 0 & 0 \end{bmatrix}^T \\ 
\vecx_1 &= \begin{bmatrix} 0 & 1 & 0 \end{bmatrix}^T \\
\vecx_2 &= \begin{bmatrix} z & \sqrt{1-z^2} & 0 \end{bmatrix}^T \\
\vecx_3 &= \begin{bmatrix} 2z & 0 & \sqrt{1-4z^2} \end{bmatrix}^T 
\end{align*}

All variables have unit norm and we wish to choose the $2$-subset of $\{\vecx_1,\vecx_2,\vecx_3\}$ that best estimates $\vecy$. Since $R^2_1 = 0$, $R^2_2 = z^2$, and $R^2_3 = 4z^2$, $\vecx_3$ will be selected first ($S_1^G = \{3\}$) if $z > 0$. $\vecx_2$ will be chosen next ($S_2^G = \{3,2\}$), and solving for $R^2$ for this pair,
\begin{align*}
R^2_{3,2} &= (\vecy^T \matX_{3,2}) (\matX_{3,2}^T \matX_{3,2})^{-1} (\matX_{3,2}^T \vecy) \\ 
&= \frac{1}{1-4z^4}\begin{bmatrix}2z & z \end{bmatrix}
\begin{bmatrix}1 & -2z^2 \\ -2z^2 & 1
\end{bmatrix} \begin{bmatrix}2z \\ z \end{bmatrix} = \frac{5z^2 - 8z^4}{1 - 4z^4} ,
\end{align*}

which goes to zero as $z \rightarrow 0^+$. However, $\vecy = -\frac{\sqrt{1-z^2}}{z} \vecx_1 + \frac{1}{z}\vecx_2$ which makes $R^2_{1,2} = 1$ for the optimal set $\{\vecx_1, \vecx_2\}$ ($S_2 = \{1,2\}$).

\section{Greedy Selection for GLMs}\label{sec:glmAppendix}
In this section, we state guarantees for feature selection in sparse generalized linear regression using the framework developed in Section~\ref{sec:derivation}. 
For introducing sparsity, a relevant regularizer (such as $\ell_1$) is often used. An alternative to regularization is to apply Algorithms~\ref{alg:oblivious}~--~\ref{alg:omp} to the log-likelihood function. To use guarantees presented in Section~\ref{sec:derivation}, we derive sample complexity conditions on the design matrix $\bX$ that are sufficient to bound the submodularity ratio $\gamma_{\sfU,k}$ with high probability.

Recall that Theorems~\ref{thm:obliviousGuarantee}~--~\ref{thm:greedyGuaranteeExtra} require strong convexity and bounded smoothness on a sparse support. While in general GLMs are not strongly convex, nor do they have bounded smoothness, it can be shown that on the \emph{restricted} set of sparse supports under some mild restrictions on the design matrix,  they possess both these traits. Moreover, note that some regularizers such as $\ell_2$ that are widely applied for GLMs automatically imply strong convexity. We present the analysis for both regularized and unregularized regression. 

For continuity, we reintroduce some notation here. We represent the data as $(\bx_i, y_i)$, where $\forall i, \bx_i \in \bbR^p$ are features, and $y_i \in \bbR$ represents the response. The log conditional can be written in its \emph{canonical} form as~\cite{Dey00}: 
\begin{equation}
\label{eq:canonical}
\log p( y | \bx; \bbeta)  = h^{-1}(\tau) y \bbeta^\top x - Z(\bbeta, \bx) + g(y, \tau),
\end{equation}
where $Z(\cdot)$ is the log partition function,  and $\bbeta, \tau$ are the  parameters ($\tau$ is also called the dispersion parameter). For $n$ observations, we can write the log-likelihood as
\begin{equation}
\label{eq:glmlogLikelihood}
l_{GLM}(\bbeta) := \sum_{i=1}^n \log p(y_i | \bx_i; \bbeta) \;.
\end{equation}

The parameters of this distribution can then be learned by maximizing the log-likelihood. Recall that equivalently we can minimize its negative.
Also, typically a regularization term is added for stability and identifiability. The \emph{loss} function $g$ to minimize for learning can be written as
\begin{equation}
g(\bbeta) := - l_{GLM}(\bbeta)  + \eta \| \bbeta\|_2^2 .
\label{eqn:GLMloss}
\end{equation}  

Similar to the \eqref{eq:mlfunction}, we can define a normalized set function $ f_{GLM1}(\cdot)$ associated with $g(\cdot)$ as
\begin{equation*}
  \max_{| \sfS | \leq k} f_{GLM1}(\sfS)\, \Leftrightarrow\, \min_{\substack{\bbeta: \bbeta_{\sfS^c = 0} \\ | \sfS | \leq k}} g(\bbeta)\;.
\end{equation*}
Note that increasing the support set of $\bbeta$ does not decrease the log-likelihood, so our set function $f_{GLM1}(\cdot)$ is indeed monotone. Further, note that for normalizing the set function, we can use $f_{GLM1}(\emptyset) = g(\mathbf{0})$.

The Hessian of $g$ at any point $\bbeta$ can be written as
\begin{equation}
\bH(\bbeta) := \frac{\partial^2 g (\bbeta)}{\partial \bbeta \partial \bbeta^\top } = \bX^\top \bD \bX + \eta \bI, 
\label{eq:glmHessian}
\end{equation} 
where $\bD$ is a diagonal matrix with $\bD_{ii} = h^{-1}(\tau) Z''(\bbeta, \bx_i)$.  
Next, we state assumptions required for the sample complexity bounds. We assume that 
$\bD_{ii}$ is upper bounded by $s$
at any value of the domain of $\bbeta$. This implies that 
$\eta\bI \preccurlyeq  \bH(\bbeta) \preccurlyeq s \bX^\top \bX + \eta \bI$
for all $\bbeta$.

Let $\sfT \subset [p]$ so that $| \sfT| \leq r$. For any vector $\bv$, define $\bv_\sfT$ be the vector formed by replacing all indices in $[p] \backslash \sfT$ of  $\bv$ by 0. let $\cP_\sfT$ be the operator that achieves this i.e. $\cP_\sfT \bv = \bv_\sfT$. 

\begin{asm}
 We make the following assumptions for Proposition~\ref{thm:bahmani}. Let the rows of $\bX$ be generated i.i.d. from some underlying distribution, so that $\bbE[\bx\bx^\top] = \bC$.  For all $\sfT \subset [p], | \sfT | \leq r$,
\begin{enumerate}
\item $\| \bx_\sfT \|_2  \leq R$, and
\item None of the matrices  $\cP_\sfT^\top \bC \cP_\sfT  $ are the zero matrix.
\end{enumerate}
 \label{asm:regularized}
 \end{asm}

Further, define $\theta_\sfT := \lambda_{\text{max}} (\cP_\sfT^\top \bC \cP_\sfT )$, let $\bar{\theta} = \max_{{\sfT \subset [p], |\sfT| \leq r  }} \theta_\sfT$, $\tilde{\theta} = \min_{{\sfT \subset [p], |\sfT| \leq r  }} \theta_\sfT$.

\begin{prop}[from ~\cite{Bahmani2013}]
With Assumption~\ref{asm:regularized}, for $ \delta \in (0,1) $ and $n > \frac{ R (\log r + r (1 + \log \frac{p}{k} - \log \delta)}{\tilde{\theta} (1 + \epsilon) \log (1+ \epsilon) -\epsilon}  $, $\lambda_{\text{max}} (\cP_\sfT ^\top\bX^\top \bX \cP_\sfT)\leq (1+\epsilon) \bar{\theta} ) $ with probability $(1 - \delta)$.
\label{thm:bahmani}
\end{prop}

\begin{cor}[Regularized GLM Sample Complexity]
Under Assumption~\ref{asm:regularized}, for $ \delta \in (0,1) $ and $n > \frac{ R (\log r + r (1 + \log \frac{p}{k} - \log \delta)}{\tilde{\theta} (1 + \epsilon) \log (1+ \epsilon) -\epsilon}  $, with probability ($1-\delta)$, the submodularity ratio of $f_{GLM1}$,  $\gamma_{\sfU,r} \geq {\eta}/{(\eta + s(1+\epsilon)\bar{\theta} })$. 
\label{thm:samplecomplexity}
\end{cor}
\begin{proof}
Note that $f_{GLM1}$ is $\eta$-strongly convex. Further, from~\eqref{eq:glmHessian}, and Prop~\ref{thm:bahmani}, it is $(\eta + s(1+\epsilon)\bar{\theta})$-smooth with probability $(1-\delta)$. The result now follows from Theorem~\ref{thm:ratioBound}.
\end{proof}
\begin{rem}
The above discussion can also be motivated by the Restricted Stability property of the Hessian of the loss function. Define: 
\begin{equation*}
\bA_r(\bbeta) := \max \{ \bv^\top \bH(\bbeta) \bv\big\vert |\text{supp}(\bbeta) \cup \text{supp}(\bv)| \leq r, \| \bv\|_2 =1  \} \; ,
\end{equation*}
\begin{equation*}
\bB_r(\bbeta) := \min \{ \bv^\top \bH(\bbeta) \bv\big\vert |\text{supp}(\bbeta) \cup \text{supp}(\bv)| \leq r, \| \bv\|_2 =1  \} \; .
\end{equation*}
The Hessian is said to be $\mu_r$-SRH (Stable Restricted Hessian)~\cite{Bahmani2013} if, $\mu_r \geq \frac{\bA_r}{\bB_r}$. It is straightforward to see that $\gamma_{\sfU,r} \geq \nicefrac{1}{\mu_r}$. 
\end{rem}

\subsection{Restricted strong convexity of GLMs}
Under stronger assumptions on the design matrix, $\bX$, it is possible to have $\gamma > 0$ even when $\eta = 0$ (i.e. the loss is unregularized) in Corollary~\ref{thm:samplecomplexity}. In the following discussion, we assume $\eta = 0$ to avoid clutter, but the discussion can be readily extended to the case when $\eta >0$. Similar to the case study of regular GLMs, we consider the GLM \emph{loss} as the negative log-likelihood, defined as:

\begin{equation}
h(\bbeta ) = - l_{GLM}(\bbeta),
\end{equation}
 and provide sample complexity bounds for weak submodularity to hold for the associated set function:

\begin{equation*}
 \max_{| \sfS | \leq k} f_{GLM2}(\sfS)\, \Leftrightarrow\, \min_{\substack{\bbeta: \bbeta_{\sfS^c = 0} \\ | \sfS | \leq k}} h(\bbeta)\;.
\end{equation*}
 
Restricted strong convexity for GLMs was studied by Negahban, 
et al.~\cite{Negahban2012journal} 
with the restricted sets being the neighborhood of the true optimum of a convex loss function. This allows GLMs to satisfy RSC on sparse models with support on the true optimum.~\cite{loh2015} extended the RSC conditions to hold uniformly for all $k$-sparse models for GLMs. We present the requisite results here. 

Recall that $D_h(\bx,\by): = h(\bx) - h(\by) - \langle\nabla  (\by), \bx -\by \rangle$. For brevity, we drop the subscript, and use $D(\bx,\by)$. Also recall that $h(\cdot)$ is $m$-strongly convex on a set $\sfS$ if $\forall \bx, \by \in \sfS$, $D(\bx,\by) \geq \frac{m}{2} \| \bx - \by \|_2^2$, and that $h(\cdot)$ is $M$-smooth on the same domain if $D(\bx,\by) \leq \frac{M}{2} \| \bx - \by \|_2^2$. Let $\bbB_m(r)$ represent an $m$-norm ball of radius $r$. 

\begin{asm}
The design matrix $\bX$ consists of samples drawn i.i.d from a sub-Gaussian distribution with parameter $\sigma^2_\bx$, and covariance matrix $\Sigma$. 
\label{asm:subgaussian}
\end{asm}

\begin{theorem}[from ~\cite{loh2015}]
If Assumption~\ref{asm:subgaussian} is true, there exists a constant $\alpha_q$ depending on the GLM family, and on $\sigma^2_\bx, \Sigma, q$ such that for all vectors $\by \in \bbB_2(3) \cap \bbB_1(q)$ for a constant $q$ s.t. $q \sqrt{\frac{\log p}{n}} \lesssim 1$, so that with probability 1 - $c_1 \exp(-c_2n)$,

 \[
    D(\bx,\by) \geq
\begin{cases}
    \frac{\alpha_q}{2}\| \Delta\|_2^2 - \frac{c^2\sigma_\bx}{2 \alpha_q} \frac{\log p}{n} \| \Delta \|^2_1,& \text{ if } \| \Delta\|_2 \leq 3 \\
    \frac{3\alpha_q}{2} \|\Delta\|_2 - 3c \sigma_\bx \sqrt{\frac{\log p}{n}} \| \Delta \|_1 ,               & \text{otherwise},
\end{cases}
\]

where $\Delta  = \bx -\by$. Similarly, $D(\bx, \by)$ can be upper bounded. Recall that $s $ is such that $max_i \,\bD_{ii} \leq s$ for $\bD$ as used in \eqref{eq:glmHessian}. Then, with probability 1 - $c_1 \exp(-c_2n)$,

\[
D(\bx, \by) \leq s \lambda_{\text{max}} (\Sigma ) \left( \frac{3}{2} \| \Delta \|_2^2 + \frac{\log p}{n} \| \Delta \|_1^2 \right).
\]

\label{thm:rscGLM}
\end{theorem}

Theorem~\ref{thm:rscGLM} can be applied to $r$-sparse sets to get the sample complexity for strong convexity.  We further assume that $\text{supp}(\bx) \subset \text{supp}(\by)$ or $\text{supp}(\by) \subset \text{supp}(\bx)$ which is not restrictive for our analysis in Section~\ref{sec:derivation}. This implies if $\bx,\by$ are $r$-sparse, $| \text{supp}(\Delta)| \leq r$.

\begin{cor}[Sample complexity sub-gaussian design] Under Assumption~\ref{asm:subgaussian}, for $n > \frac{c^2\sigma_\bx}{\alpha_q^2} { (r + | \sfU|) \log p}$, the submodularity ratio for $f_{GLM2}(\cdot)$, $\gamma_{\sfU, r} \geq \nicefrac{m'}{M'}$, where $m'=(\alpha_q - \frac{c^2 \sigma_\bx}{\alpha_q} \frac{(r+ | \sfU|) \log p}{n}) > 0$, and $M'=\lambda_{\text{max}} (\Sigma ) ( \frac{3}{2} + \frac{(r+ |\sfU|) \log p}{n})$.
\label{cor:unregGLM}
\end{cor}
\begin{proof}
Since $\| \Delta \|_1 \leq \sqrt{(r+|\sfU|)}\|\Delta\|_2 $,  from Theorem~\ref{thm:rscGLM},
  $D(\bx, \by) \geq \frac{1}{2}(\alpha_q - \frac{c^2 \sigma_\bx}{\alpha_q} \frac{(r+|\sfU|)\log p}{n})\| \Delta \|^2_2$. The sample complexity bound follows by ensuring the RHS is $>0$. This gives $h(\cdot)$ to be $m$-strongly convex with $m \geq (\alpha_q - \frac{c^2 \sigma_\bx}{\alpha_q} \frac{(r+ | \sfU|) \log p}{n})$.

Similarly, a corresponding version of restricted smoothness by using the upper bound  of $D(\bx,\by)$ in Theorem~\ref{thm:rscGLM} and using $\| \Delta \|_1 \leq \sqrt{r+|\sfU|}\|\Delta\|_2 $. The expression for the submodularity ratio then follows from Theorem~\ref{thm:ratioBound}.

\end{proof}

\bibliography{refs}

\end{document}